\documentclass{article}

     \usepackage[preprint]{neurips_2019}

\usepackage{neurips_2019}

\usepackage[utf8]{inputenc} 
\usepackage[T1]{fontenc}    
\usepackage{hyperref}       
\usepackage{url}            
\usepackage{booktabs}       
\usepackage{amsfonts}       
\usepackage{nicefrac}       
\usepackage{microtype}      
\usepackage{dsfont}
\usepackage{amsthm}

\newtheorem{rem}{\emph{Remark}}
\newtheorem{defn}{Definition}
\newtheorem{thm}{Theorem}

\newtheorem{lemma}{Lemma}
\newtheorem{assumptions}{Assumptions}
\usepackage{mathtools}
\usepackage{graphicx}
\allowdisplaybreaks
\usepackage{subcaption}

\title{Understanding the Behaviour of the Empirical Cross-Entropy Beyond the Training Distribution}

\author{%
  Matias Vera \\
  Facultad de Ingenier\'ia \\
  Universidad de Buenos Aires, Argentina\\
  \texttt{mvera@fi.uba.ar} 
  \And
 Pablo Piantanida \\
  Mila and Université de Montr\'al, Canada \\
 L2S and CentraleSup\'elec-CNRS-UPSud, France\\
   \texttt{piantani@mila.quebec} 
   \AND
   Leonardo Rey Vega \\
  Facultad de Ingenier\'ia \\
  Universidad de Buenos Aires, Argentina\\
  \texttt{lrey@fi.uba.ar} 
}

\begin{document}

\maketitle

\begin{abstract}
Machine learning theory has mostly focused on generalization to samples from the same distribution as the training data. Whereas a better understanding of generalization beyond the training distribution where the observed distribution changes is also fundamentally important to achieve a more powerful form of generalization. In this paper, we attempt to study  through the lens of information measures how a particular architecture behaves when the true probability law of the samples is potentially  different at training and testing times. Our main result is that the testing gap between the empirical cross-entropy and its statistical expectation (measured with respect to the testing probability law) can be bounded with high probability by the mutual information between the input testing samples and the corresponding representations, generated by the encoder obtained at training time. These  results of theoretical nature are supported by numerical simulations showing that the mentioned mutual information is representative of the testing gap, capturing qualitatively the dynamic in terms of the hyperparameters of the network.
\end{abstract}

\section{Introduction}

Most theories of generalization for classification, both theoretical and empirical, assumes that models are trained and tested using data drawn from some fixed distribution.  What is often needed in practice, however, is to learn a classifier which performs well on a target domain with a significantly different distribution of the training data, which may involve similar concepts that were observed  previously by the learner but with some features changes~\citep{DBLP:journals/corr/abs-1901-10912}. In general, we would expect that the result of the learning phase to be able to generalize well to other distributions, with the added functionality of being able to monitoring the dynamics of the generalization due to external factors or non-stationarities of the involved distributions. In this paper,   we investigate  through the lens of information-theoretic principles how the testing gap between the empirical cross-entropy and its statistical expectation (computed from the target distribution) behaves when the true probability law of the data  samples is potentially different at training and testing time.

According to classical statistical learning theory~\citep{boucheron_bousquet_lugosi_2005}, models with many parameters tend to \emph{overfit} by representing the training data too accurately, therefore diminishing their ability to generalize to unseen data \cite{bishop_praml}. Interestingly enough,  this phenomena does seem to be happening with Deep Neural Networks (DNNs) which even with many parameters and a modest number of training samples present good generalization properties as shown by ~\cite{Neyshabur2017GeometryOO,DBLP:journals/corr/ZhangBHRV16}. 


Stochastic representations encompass the classical learning problem to include graphic models, even neuronal ones such as Variational Auto-Encoders (VAEs)~\citep{kingma2013auto} or Restricted Boltzmann Machines (RBMs)~\citep{Hinton_guia}. With this models, mutual information between feature inputs and their representations becomes a relevant quantity. Empirical studies based on the Information Bottleneck (IB) method \citep{Tishby1999information} have shown that this mutual information may be related to the overfitting problem in DNNs. The IB method studies the tradeoff between accuracy and information complexity measured in terms of the mutual information. Statistical rates on the empirical estimates of the corresponding IB tradeoffs have been reported in~\cite{Shamir:2010:LGI:1808343.1808503}. \citet{DBLP:journals/corr/Shwartz-ZivT17} show, based on empirical evidence, that the DNN generalization capacity is induced by an implicit data compression of the feature inputs. However, these trade-offs are still controversial and the question remains open. As a matter of fact, recent works by~\citep{DBLP:journals/corr/abs-1802-09766} and~\citep{saxe2018} report difficulties in using the IB trade-off as an information-regularized objective for training. Indeed,  it is known that estimating mutual information from high-dimensional data samples is a difficult and challenging problem for which, even state-of-the-art methods, might lead to misleading numerical results~\citep{DBLP:conf/icml/BelghaziBROBHC18}. 



\subsection{Our Contribution}

We investigate the probability concentration of the testing gap between the empirical cross-entropy and its statistical expectation, measured with respect to the testing probability distribution which may be different of the training distribution. More specifically, there are two attributes that characterize our model: (a) we study a deviation on the target (testing) dataset once the training stage is finished, i.e., for a given soft-classifier; and (b) we consider randomized encoders which allows us not only to study deterministic feed-forward structures, but also stochastic models.

Theorem~\ref{thm:regularizar} provides the rigorous statement of the cross-entropy deviation bound that is the basis of this work. In contrast to standard learning concentration inequalities, this bound  scales with $\log(n)/\sqrt{n}$ and $1/\sqrt{n}$ in a $n$-length dataset. Our bound depends on several factors: a mutual information term between the input and the representations generated from them using a reference selected encoder; a second term that measures the decoder efficiency; and two other magnitudes that measure how robust can be interpreted the problem, motivated by \cite{Xu2012} among others. Despite the fact that our results may not lead to the tightest bounds, they are intended to reflect the importance of information-theoretic concepts in the problem of representation learning and the different trade-offs that can be established between information measures and quantities of interest in statistical learning. 

An empirical investigation of the interplay between the cross-entropy deviation and the mutual information is provided on high-dimensional datasets of natural images with rotations and translations on the target domain. These simulations show the ability of mutual information to predict the behavior of the gap for the case of three well-known stochastic representation models: (a) the standard Variational Auto-Encoders (VAE)~\citep{kingma2013auto}; (b) the log-normal encoder presented in the Information Dropout scheme in~\citep{2016arXiv161101353A}; and (c) the classical encoder based on Restricted Boltzmann Machines (RBMs)~\citep{Hinton_guia}.  These results validates the fact that mutual information is an important quantity that is strongly related to the generalization properties of learning algorithms, motivating the need of further studies in this respect. 

The rest of the paper is organized as follows. In Section \ref{Section-2}, we introduce a stochastic model which combines the randomized encoder concept like in \cite{2016arXiv161101353A} with the possibility of decomposing the classifier in a encoder and a decoder (\cite{DBLP:journals/corr/Shwartz-ZivT17}).  In Section  \ref{sec:mainbound} we present our general concentration inequality based on the above mentioned mutual information and other specific magnitudes, which will be discussed. While in Section \ref{sec:experimentos} we show numerical evidence for some selected models, in Section \ref{sec:conclusion} we provide concluding remarks. Major mathematical details 
are relegated to the appendices in the Appendix.

\section{Representation and Statistical Learning}\label{Section-2}

We are interested in the problem of pattern classification consisting in the prediction of the unknown class that match an observation. An observation or example is often a sample $x\in \mathcal{X}$ which have an associated label $y\in \mathcal{Y}$ (finite space). An $|\mathcal{Y}|$-ary classifier is defined by a (stochastic) decision rule $Q_{\hat{Y}|X}:\mathcal{X} \rightarrow \mathcal{P}(\mathcal{Y})$, where $\hat{Y}\in\mathcal{Y}$ denotes the random variable associated to the classifier output and $X$ is the random observed example. Typically, this classifier is trained with samples generated according to an unknown training distribution. It is also assumed that the testing examples and their corresponding labels are generated in an i.i.d. fashion according to $p_{XY}\coloneqq p_XP_{Y|X}$ (which could be possibly different from the above mentioned training distribution).\footnote{In writing this we allow us a little abuse of notation to simplify it since $p_X$ is a pdf. and $P_{Y|X}$ is pmf.}

The problem of finding a good classifier can be  divided into that of simultaneously finding an (possibly randomized) encoder $q_{U|X}:\mathcal{X}\rightarrow\mathcal{P}(\mathcal{U})$ that maps raw data to a higher-dimensional (feature) space $\mathcal{U}$ and a soft-decoder $Q_{\hat{Y}|U}:\mathcal{U}\rightarrow\mathcal{P}(\mathcal{Y})$ which maps the representation to a probability distribution on the label space $\mathcal{Y}$. These  mappings induce an equivalent classifier: 
\begin{equation}
Q_{\hat{Y}|X}(y|x) = \mathbb{E}_{q_{U|X}}\left[Q_{\hat{Y}|U} (y|U)|X=x\right],  \label{eq-clasifier}
\end{equation}
\begin{rem}
	In the standard methodology with deep representations, we consider $L$ randomized encoders ($L$ layers) $\{q_{U_l|U_{l-1}}\}_{l=1}^L$ with $U_0\equiv X$. Although this appears at first to be more general, it can be casted formally using the one-layer case formulation  induced by the marginal distribution that relates  the input and the final $L$-th output layer. Therefore results on the one-layer formulation also apply to the $L$-th layer formulation and thus, we shall thus focus on the one-layer case without loss of generality.  
\end{rem}
This representation contains several cases of interest as the feed-forward neural net case as well as genuinely graphical model cases such as VAE or RBM. The computation of \eqref{eq-clasifier} requires marginalizing out $u \in\mathcal{U}$ which could be computationally prohibitive in practice. We use the  \emph{cross-entropy} as a loss-function:
\begin{equation}
\ell(x,y)\coloneqq\ell\big(q_{U|X}(\cdot|x),{Q}_{\hat{Y}|U}(y | \cdot)\big)=\mathbb{E}_{q_{U|X}} \left[ -\log Q_{\hat{Y}|U}(y|U)|X=x\right]. \label{eq-true-loss}
\end{equation}

The learner's goal is  to select $(q_{U|X},Q_{\hat{Y}|U})$ by minimizing the expected risk under the training distribution: ${\mathcal{L}}(q_{U|X},Q_{\hat{Y}|U}) \coloneqq  \mathbb{E}_{p_{XY}}   \left[\ell(X,Y)\right]$. This is done at training time using the empirical risk computed with i.i.d. samples from the training distribution. After training phase is over we would like to evaluate the performance of the obtained $(q_{U|X},Q_{\hat{Y}|U})$ based on a testing dataset: $\mathcal{S}_n \coloneqq \{(x_1,y_1)\cdots(x_n,y_n) \}$ which is independent from the training set (but not necessarily sampled from the same distribution).  The \emph{testing} risk is defined by
\begin{equation}
{\mathcal{L}}_{\text{emp}}(q_{U|X},Q_{\hat{Y}|U},{\mathcal{S}_n} )\coloneqq  
\frac1n\sum\limits_{i=1}^n \ell(x_i,y_i).\label{eq-EM-def}
\end{equation}
Since the testing risk is evaluated on finite size samples, its evaluation may be sensitive to sampling noise error. 
The gap, to be defined next, is a measure of how an encoder-decoder pair could perform on unseen data (at training time) contained in the testing set $\mathcal{S}_n$.
\begin{defn}[Error gap]
	Given a stochastic encoder $q_{U|X}:\mathcal{X}\rightarrow\mathcal{P}(\mathcal{U})$ and decoder  $Q_{\hat{Y}|U}:\mathcal{U}\rightarrow\mathcal{P}(\mathcal{Y})$, the \emph{error gap} is given by 
	\begin{equation}
	\mathcal{E}_{\textrm{gap}}(q_{U|X},Q_{\hat{Y}|U},{\mathcal{S}_n}) \coloneqq\left| \mathcal{L}_{\text{emp}}(q_{U|X},Q_{\hat{Y}|U},{\mathcal{S}_n}) - \mathcal{L}(q_{U|X},Q_{\hat{Y}|U}) \right|,  \label{def-gap}
	\end{equation}
	which quantifies the error associated to $(q_{U|X},Q_{\hat{Y}|U})$ when $\mathcal{L}_{\text{emp}}({q_{U|X},Q_{\hat{Y}|U}},{\mathcal{S}_n} )$ is considered an estimate of $\mathcal{L}(q_{U|X},Q_{\hat{Y}|U})$ (calculated with the testing distribution). 
\end{defn}
\begin{rem}
	The definition of this gap should not be confused with the typical one used for PAC-style bounds \citep{Devroye97a}. The latter is computed w.r.t. the  training dataset while here we will study a deviation bound on the testing set, i.e., after the training stage has been accomplished. Our focus is reasonable for scenarios where the testing statistics evolves over time and may not match the training distribution.
\end{rem}

\section{Information Theoretic Bounds on the Gap}\label{sec:mainbound}

In this section, we first present our main result in Theorem~\ref{thm:regularizar}, which is a bound on the  gap~\eqref{def-gap} with probability at least $1-\delta$, as a function of a fixed randomized encoder-decoder pair $(q_{U|X},Q_{\hat{Y}|U})$. In particular, we show that the mutual information between the input raw data and its representation controls the gap with a scaling $\mathcal{O}\left(\frac{\log(n)}{\sqrt{n}}\right)$, which leads to a so-called informational deviation error bound. The information measures to be used in the work are~\cite{cover}: \emph{Kullback-Leibler (KL) divergence} $\mathcal{D}( {p}_{X}\| {q}_{X})\coloneqq\mathbb{E}_{p_{X}}\left[\log \frac{p_{X}(X)}{q_X(X)}\right]$; \emph{conditional KL divergence} $
\mathcal{D}( {p}_{Y|X}\| {q}_{Y|X}| {p}_{X})\coloneqq \mathbb{E}_{ {p_X}}\left[ \mathcal{D}\big( {p}_{Y|X}(\cdot| X)\| {q}_{Y|X}(\cdot| X)\big) \right]$ and \emph{mutual information} $\mathcal{I}(p_X;{p}_{Y|X})\coloneqq\mathcal{D}( {p}_{Y|X} \| {{p}_Y} | p_X)$. In this section, we will make use of the following assumptions:
\begin{assumptions}\label{asumption1}
	We assume $\mathcal{X}=\text{Supp}(p_X)$ and $P_{Y}(y_{\min}) \coloneqq \min_{y\in\mathcal{Y}} P_Y(y)>0$ without loss of generality because we can ignore the zero probability events. We assume that $\text{Vol}(\mathcal{U})<\infty$\footnote{The output encoder alphabet could be continuous. With this we are implying that it is bounded set. If $\mathcal{U}$ is discrete this assumption means that it is of finite cardinality.} and that the selected encoder-decoder pair $(q_{U|X},Q_{\hat{Y}|U})$ is such that $\displaystyle Q_{\hat{Y}|U}(y_{\min}|u_{\min})\coloneqq\inf_{\substack{y\in\mathcal{Y}\\u\in\mathcal{U}}}  Q_{\hat{Y}|U}(y|u)\geq \eta>0$. 
\end{assumptions} 

The assumption of $Q_{\hat{Y}|U}(y_{\min}|u_{\min})>0$ is typically valid for the soft-max decoder, since in practice its parameters never diverge. 

\begin{thm}[Information-theoretic bound]\label{thm:regularizar}
	For every $\delta\in(0,1)$, with probability at least $1-\delta$ over the choice of $\mathcal{S}_n\sim p_{XY}$,  the gap satisfies:
	\begin{align}
	\mathcal{E}_{\textrm{gap}}(q_{U|X},Q_{\hat{Y}|U},\mathcal{S}_n)&\leq\inf_{K\in\mathbb{N}}2\epsilon(K)+  A_\delta\sqrt{\mathcal{I}(p_X;q_{U|X})}\cdot\frac{\log(n)}{\sqrt{n}}r(K)\nonumber\\
	&\hspace{-0.1cm}+\frac{D_\delta\cdot \mathcal{D}_{\text{HL}}\left(Q^D_{Y|U}\|Q_{\hat{Y}|U}|q^D_U\right)+C_\delta} {\sqrt{n}}+\mathcal{O}\left(\frac{\log(n)}{n}\right),\label{eq-main-bound-GG}
	\end{align} 
	$\forall\,(q_{U|X},Q_{\hat{Y}|U})$ that meets Assumptions \ref{asumption1}, where $\mathcal{D}_{\text{HL}}$ is the Hellinger distance
	\begin{equation}\label{eq:hellinger}
	\mathcal{D}_{\text{HL}}\left(Q^D_{Y|U}\|Q_{\hat{Y}|U}|q^D_U\right)=\sqrt{\frac{1}{2}\cdot\mathbb{E}_{q^D_{U}}\left[\sum_{y\in\mathcal{Y}}\left(\sqrt{Q_{\hat{Y}|U}(y|U)}-\sqrt{Q^D_{Y|U}(y|U)}\right)^2\right]},
	\end{equation} 
	constants are defined as
	$A_\delta\coloneqq\sqrt{2}B_\delta$, $B_\delta\coloneqq\Big(1+\sqrt{\log\big(\frac{|\mathcal{Y}|+4}{\delta}\big)}\Big)$, $C_\delta\coloneqq 2\text{Vol}\left(\mathcal{U}\right)e^{-1}$ $+B_\delta\sqrt{|\mathcal{Y}|} \log \left(\frac{\text{Vol}\left(\mathcal{U}\right)}{P_{Y}(y_{\min})}\right)$, $D_\delta=Q_{\hat{Y}|U}^{-1/4}(y_{\min}|u_{\min})\sqrt{8\frac{|\mathcal{Y}|+4}{\delta}}$; and
	\begin{equation}
	\epsilon(K)=\sup_{\substack{k,x,y:\\1\leq k\leq K\\y\in\mathcal{Y}\\x\in\mathcal{K}_k^{(y)}}}\left|\ell(x,y)-\ell(x^{(k,y)},y)\right|,\qquad\qquad r(K)=\frac{1}{\displaystyle\min_{\substack{k,y:\\1\leq k\leq K\\y\in\mathcal{Y}}}\int_{\mathcal{K}_k^{(y)}}p_X(x)dx}.\label{eq:epsilon_r_definitions}
	\end{equation}
	where $\left(\{\mathcal{K}_k^{(y)}\}_{k=1}^K,\{x^{(k,y)}\}_{k=1}^K\right)_{y\in\mathcal{Y}}$ are $|\mathcal{Y}|$ partitions of $\mathcal{X}$ and their respective centroids. They are clearly functions of the natural number $K$ and such that for each $y\in\mathcal{Y}$:  $\bigcup_{k=1}^K\mathcal{K}_k^{(y)}=\mathcal{X},\;\mathcal{K}_{i}^{(y)}\cap\mathcal{K}_{j}^{(y)}=\emptyset\;\forall1\leq i<j\leq K,\;\text{Vol}(\mathcal{K}_k^{(y)})>0\;\forall 1\leq k\leq K$; and
	\begin{equation}
	q_U^D(u)=\sum_{k=1}^K\sum_{y\in\mathcal{Y}}q_{U|X}(u|x^{(k,y)})\int_{\mathcal{K}_k^{(y)}}p_{XY}(x,y)dx, 
	\end{equation}	
	\begin{equation}
	Q_{Y|U}^D(y|u)=\frac{\sum_{k=1}^Kq_{U|X}(u|x^{(k,y)})\int_{\mathcal{K}_k^{(y)}}p_{XY}(x,y)dx}{q^D_U(u)}\label{eq:Qy|u_D}
	\end{equation}
	are distributions functions induced by the quantization of the testing distribution $p_{XY}(x,y)$ by the above mentioned partitions.
\end{thm}

The proof is relegated to the Appendix. This bound has some important terms which are worth analyzing:
\begin{itemize}
    \item Encoder and decoder $(q_{U|X},Q_{\hat{Y}|U})$:  the encoder and the decoder  in Theorem \ref{thm:regularizar} are considered to be given. However, without loss of generality they can correspond to the learned encoder and decoder  during training time. In this case, the bound in~\eqref{eq-main-bound-GG} should be further averaged with respect to the randomness of the training samples --according to the training distribution-- and the training algorithm itself.
	\item $\mathcal{I}(p_X;q_{U|X})$: Mutual information between raw data $X$ and its randomized representation $U$ is a regularization term used to reduce the overfitting, that some authors understand as ``measure of information complexity'' \citep{2016arXiv161101353A, DBLP:journals/corr/AlemiFD016,8379424}. Theo. \ref{thm:regularizar} is a first step in order to explain how and why this effect happens. This term presents a scaling rate of $n^{-1/2}\log(n)$ and it is the most important term of our deviation bound. In Section \ref{sec:experimentos} we present a empirical analysis in order to study its importance.
	\item $\mathcal{D}_{\text{HL}}\left(Q^D_{Y|U}\|Q_{\hat{Y}|U}\right)$: Hellinger distance  could be seen as a measure of the decoder efficiency in comparison with the decoder $Q^D_{Y|U}$, which is induced by the randomized encoder $q_{U|X}$ and the quantized testing distribution. When $Q_{\hat{Y}|U}=Q^D_{Y|U}$ this term is zero, suggesting that this selection minimizes the error gap. In this way, we are interested in decoders with sufficient freedom degrees as the soft-max decoder.
	\item $\epsilon(K)$ and $r(K)$: Motivated by robust quantization \citep{Xu2012}, these functions define, for each $y\in\mathcal{Y}$, an artificial discretization of $\mathcal{X}$ space into cells (partition element). This discretization allow us to introduce some information-theoretic techniques and results for discrete alphabets during the proof of our result. While $\epsilon(K)$ is associated with the robustness of the loss function over the partition element, $r(K)$ is the minimum probability of falling into a cell. There is a tradeoff between these: while $\epsilon(K)$ is a decreasing function (when the number of cells is increased, they may be smaller), $r(K)$ is a increasing one (smaller cells enclose less probability).
	\item $Q_{\hat{Y}|U}(y_{\min}|u_{\min})$: The maximum of the loss function value could be a poor choice to generate a deviation bound with dependence on the decoder. In our case, a more sensible approach is followed by using the Hellinger distance between $Q_{\hat{Y}|U}$ and $Q^D_{Y|U}$. It can be seen that this decoder dependent term will not be relevant in two scenarios: when decoder selection $Q_{\hat{Y}|U}$ is close to $Q^D_{Y|U}$ (in a Hellinger distance sense) and when $Q_{\hat{Y}|U}(y_{\min}|u_{\min})$ is small enough (through $D_\delta$). The consideration of these two possibilities and the $1/\sqrt{n}$ scaling of this term could justify to disregarding it when the number of samples is large enough.  
	\item $\text{Vol}(\mathcal{U})$: Note that if ReLU activations are implemented (whose volume is limited for bounded entries), $\text{Vol}(\mathcal{U})$ is expected to be larger than for the case of sigmoid activations. As a consequence, the mutual information will have a major influence in the generalization with saturated activations. This observation matches with \citet{saxe2018} analysis.
\end{itemize}

\section{Experimental Results}
\label{sec:experimentos}

\begin{figure}
	\centering
	\includegraphics[scale=0.48]{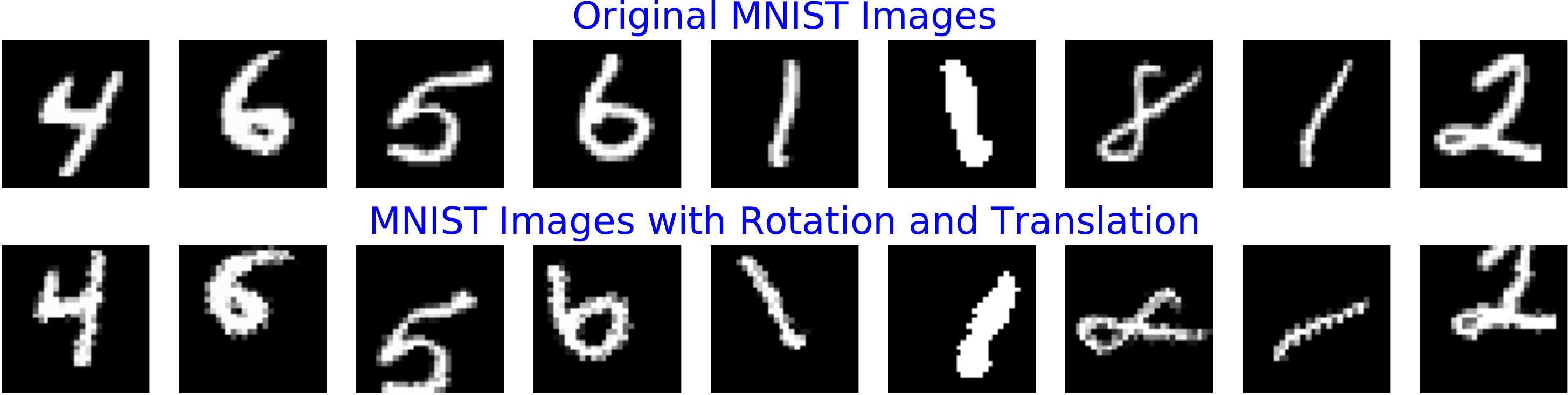}
	\caption{Comparison between original MNIST dataset and our propose with rotations and translations}
	\label{fig:mnistperturbada}	
\end{figure}

In this section, we experimentally check our bound in some stochastic models used in practice. We will show that the mutual information is representative of the gap behavior, i.e. we compare, after a training stage, a $\mathcal{E}_{\textrm{gap}}(q_{U|X},Q_{\hat{Y}|U},\mathcal{S}_n)$ quantile (the magnitude to be bounded) with $A\cdot\sqrt{\mathcal{I}(p_X;q_{U|X})}+C$, where $A$ and $C$ are universal constants representative of the corresponding ones in Theorem \ref{thm:regularizar}. The magnitudes are compared for several rules $Q$ indexed by a Lagrange multiplier defined in each experiment. That is, we first select an encoder $q_{U|X}$ and a decoder $Q_{\hat{Y}|U}$ based on a Lagrange multiplier $\lambda$ in the cost function and the training set, and then we evaluate the risk deviation based on independent features to those of used during the training, i.e.,  according to the testing dataset $\mathcal{S}_n$. We focus on the qualitative behavior, so we plot with different axes for the quantile gap and mutual information in order to get rid of the mentioned constants.

We will show that the mutual information is representative of the behavior of said gap, even when the distribution with which the test samples are generated does not match with the training law. As our main goal is not to present a new classification methodology, with competitive  results with state-of-the-art methods, we restrict ourself to use small databases, as motivated in \citet{Neyshabur2017GeometryOO} work. We sample a random subset of MNIST (standard dataset of handwritten digits). The size of the training set is $5K$ and the algorithms will be tested with standard MNIST dataset and with a disturbed version with translations and rotations. Random translations are drawn from an uniform distribution between -5 and 5 (quantized) for each axis and random rotations are drawn from an uniform distribution $(-\frac{\pi}{4},\frac{\pi}{4})$ for the angle, as can be seen in Fig. \ref{fig:mnistperturbada}. Experiments with CIFAR-10 dataset (natural images) can be seen in Appendix. 

In most applications, alphabets are such that $\mathcal{X}\subset\mathbb{R}^d$,  $\mathcal{U}\subset\mathbb{R}^m$, where $d$ is the number of input units and $m$ is the number of hidden units of the auxiliary variable, so we refer to the vector random variables $\mathbf{X}, \mathbf{U}$ respectively. We approximate $\mathcal{L}(q_{U|X},Q_{\hat{Y}|U})$ with a $5K$ dataset and $\mathcal{L}_{\text{emp}}(q_{U|X},Q_{\hat{Y}|U},{\mathcal{S}_n})$ with different independent mini-testing datasets of $100$ samples, i.e., using the rest of the features. The $\mathcal{E}_{\textrm{gap}}(q_{U|X},Q_{\hat{Y}|U},\mathcal{S}_n)$ $0.95$-quantile ($\delta=0.05$) is computed based on the different values of the testing risk: $\mathcal{L}_{\text{emp}}(q_{U|X},Q_{\hat{Y}|U},{\mathcal{S}_n})$. Finally, there exists the difficulty of implementing a mutual information estimator. To this end, we make use of  the variational bound \citep{cover} to upper bound the mutual information by  $\mathcal{I}\left(p_\mathbf{X};q_{\mathbf{U}| \mathbf{X}}\right) \leq\mathcal{D}\left(q_{\mathbf{U}|\mathbf{X}} \big\| \tilde{q}_{\mathbf{U}}\big |p_{\mathbf{X}} \right)$, 
where $\tilde{q}_{\mathbf{U}}$ is an auxiliary prior pdf \citep{kingma2013auto,2016arXiv161101353A}. Consider a distribution $\tilde{q}_U$ a product distribution $\tilde{q}_{\mathbf{U}}(\mathbf{u})=\prod_{j=1}^{m}\tilde{q}_{U_j}(u_j)$. It is straightforward to check that 
\begin{equation}
\sqrt{\mathcal{I}\left(p_{\mathbf{X}};q_{\mathbf{U}| \mathbf{X}}\right)} \leq\sqrt{\sum_{j=1}^{m} \mathcal{D}\left(q_{U_j|\mathbf{X}}(\cdot|\mathbf{X}) \big\| \tilde{q}_{U_j}\big| p_{\mathbf{X}} \right)}.\!\!
\label{eq:inf_radius_opt}
\end{equation}
We will implement a parametric estimation of the KL divergence $\mathcal{D}\left(q_{U_j|\mathbf{X}}(\cdot|\mathbf{X}) \big\| \tilde{q}_{U_j}\big| \hat{P}_{\mathbf{X}} \right)$, where $\hat{P}_\mathbf{X}$ is the empirical pmf (KL estimation is an average over the sample), for each of the following architectures: normal encoder / normal prior \cite{kingma2013auto}, log-normal encoder / log-normal prior \cite{2016arXiv161101353A} and RBM encoder / $\frac{1}{n}\sum_{i=1}^nq_{U_j|\mathbf{X}}(u_j|\mathbf{x}_i)$ prior \cite{Hinton_guia}. We will refer to each of the examples by their encoders: Normal, log-Normal and RBM, respectively. Note that the mutual information does not depend on the decoder, so we use always a \emph{soft-max} output layer for simplicity. Experimental details can be seen in Appendix.


\subsection{Normal Encoder: Variational Classifier}\label{sec:normal_encoder}

\begin{figure}
	\centering
	\begin{subfigure}{.49\textwidth}
		\centering
		\includegraphics[scale=0.35]{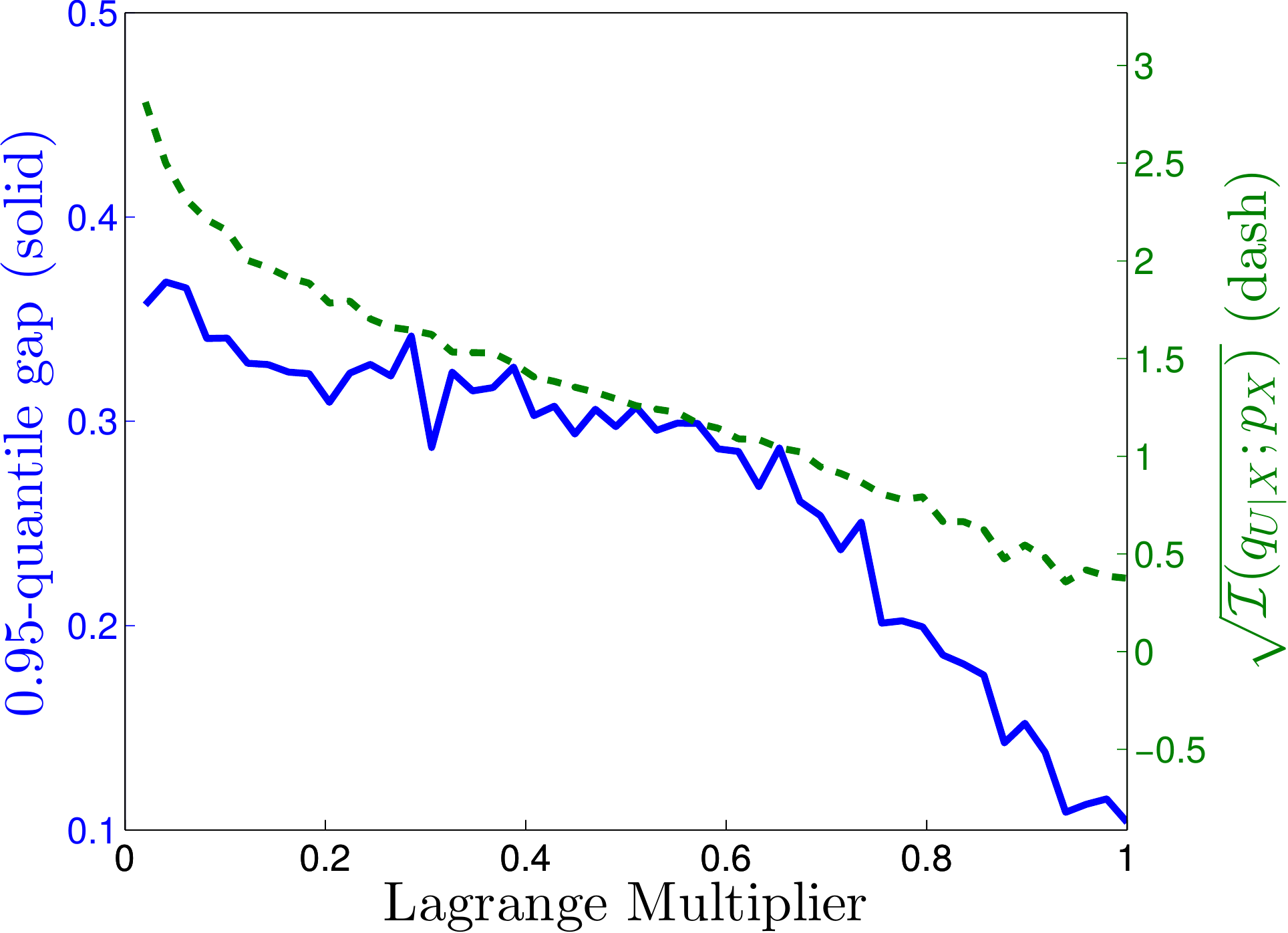}
		\caption{}
		\label{fig:normalmnist}
	\end{subfigure}
	\begin{subfigure}{.49\textwidth}
		\centering
		\includegraphics[scale=0.35]{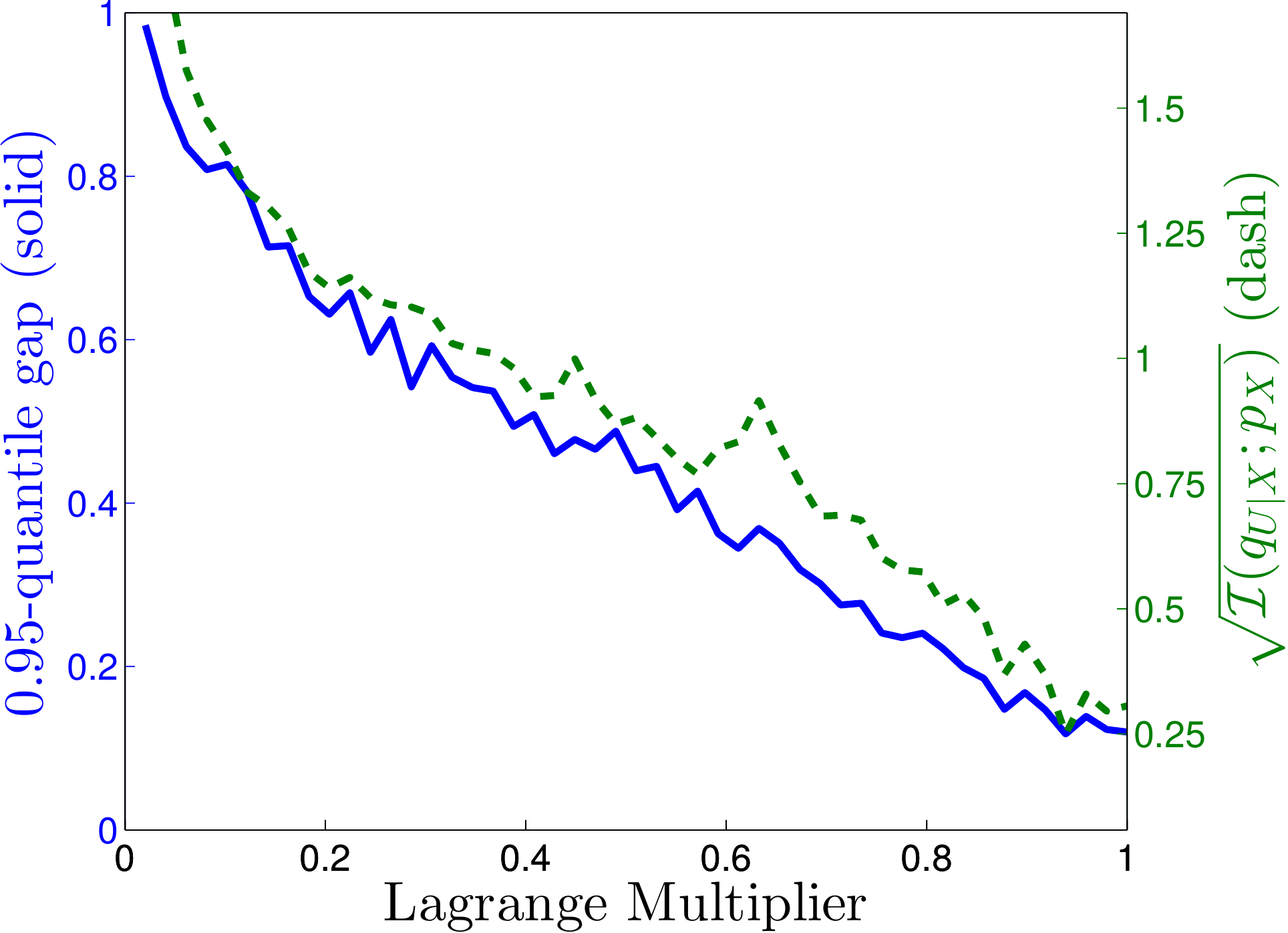}
		\caption{}
		\label{fig:normalnew}
	\end{subfigure}
	\caption{Comparison between $0.95$-quantile of $\mathcal{E}_{\textrm{gap}}(q_{U|X},Q_{\hat{Y}|U},\mathcal{S}_n)$ and the mutual information variational bound \eqref{eq:inf_radius_opt} for normal encoder and testing with: (a) Images generated with the training distribution,  (b) Images generated with other distribution.}
\end{figure}

Gaussian Variational Auto-Encoders (VAEs) introduce a normal encoder $U_j|_{\mathbf{X}=\mathbf{x}}\sim\mathcal{N}(\mu_j(\mathbf{x}),\sigma^2_j(\mathbf{x}))$ $j=[1:m]$, where $\mu_j(\mathbf{x})$ and $\log\sigma_j^2(\mathbf{x})$ are constructed via DNNs (vectorized), a standard normal prior $\tilde{U}_j\sim\mathcal{N}(0,1)$ and the decoder input is generated by simple sampling using the reparameterization trick \cite{kingma2013auto}. In this case each KL divergence in \eqref{eq:inf_radius_opt} can be estimated as follows:
\begin{equation}
\mathcal{D}\left(q_{U_j|\mathbf{X}}(\cdot|\mathbf{X}) \big\| \tilde{q}_{U_j}\big| \hat{P}_{\mathbf{X}} \right)=\frac{1}{2n}\sum_{i=1}^n\left(-\log\sigma^2_j(\mathbf{x}_i)+\sigma_j^2(\mathbf{x}_i)+\mu_j^2(\mathbf{x}_i)-1\right).
\end{equation} 
Fig. \ref{fig:normalmnist} and \ref{fig:normalnew} show $0.95$-quantile  of $\mathcal{E}_{\textrm{gap}}(q_{U|X},Q_{\hat{Y}|U},\mathcal{S}_n)$ and the mutual information variational bound \eqref{eq:inf_radius_opt} as a function of the training Lagrange multiplier $\lambda$, testing with the training and disturbed MNIST dataset respectively. There is an decreasing tendency indicating the vanishing of possible overfitting when regularization is incremented. Complex tasks, such as testing with images sampled with another distribution Fig. \ref{fig:normalnew}, generate an error gap behavior closer to the mutual information.

\subsection{Log-Normal Encoder: Information Dropout}

Information dropout propose log-normal encoders $U_j=f_j(\mathbf{X})e^{\alpha_j(\mathbf{X})Z}$ $j=[1:m]$ where $Z\sim\mathcal{N}(0,1)$, where $f_j(\mathbf{x})$ and $\alpha_j^2(\mathbf{x})$ are constructed via DNNs (vectorized) and the decoder input is generated by simple sampling using the reparameterization trick. In \cite{2016arXiv161101353A}, authors recommend to use a log-normal prior $\tilde{U}_j\sim\log\mathcal{N}(\mu_j,\sigma_j^2)$ when $\mathbf{f}(\mathbf{x})=[f_1(\mathbf{x}),\cdots f_m(\mathbf{x})]$ is a DNN with soft-plus activation, where $\mu_j$ and $\sigma_j$ are variables to train.

As $U_j|_{\mathbf{X}=\mathbf{x}}\sim\log\mathcal{N}(\log f_j(\mathbf{x}),\alpha^2_j(\mathbf{x}))$ and the KL divergence is invariant under reparametrizations, the divergence between two log-normal distributions is equal to the divergence between the corresponding normal distributions. Therefore, using the formula for the KL
divergence of normal random variables \cite{cover}, we obtain
\begin{align}
\mathcal{D}\left(q_{U_j|\mathbf{X}}(\cdot|\mathbf{X}) \big\| \tilde{q}_{U_j}\big| \hat{P}_{\mathbf{X}} \right)&=\frac{1}{n}\sum_{i=1}^n\mathcal{D}\left(\mathcal{N}(\log f_j(\mathbf{x}_i),\alpha_j^2(\mathbf{x}_i))\|\mathcal{N}(\mu_j,\sigma_j^2)\right)\\
&=\frac{1}{n}\sum_{i=1}^n\frac{\alpha_j^2(\mathbf{x}_i)+(\log(f_j(\mathbf{x}_i))-\mu_j)^2}{2\sigma_j^2}-\log\frac{\alpha_j(\mathbf{x}_i)}{\sigma_j}-\frac{1}{2}.
\end{align} 


Fig. \ref{fig:lognormalmnist} and \ref{fig:lognormalnew} show $0.95$-quantile  of $\mathcal{E}_{\textrm{gap}}(q_{U|X},Q_{\hat{Y}|U},\mathcal{S}_n)$ and the mutual information variational bound \eqref{eq:inf_radius_opt} as a function of the training Lagrange multiplier $\lambda$, testing with the training and disturbed MNIST dataset respectively. Again, there is an decreasing tendency indicating the vanishing of possible overfitting when regularization is incremented. In this case the behaviors are pretty close because the prior distribution is also trained.

\begin{figure}
	\centering
	\begin{subfigure}{.49\textwidth}
		\centering
		\includegraphics[scale=0.35]{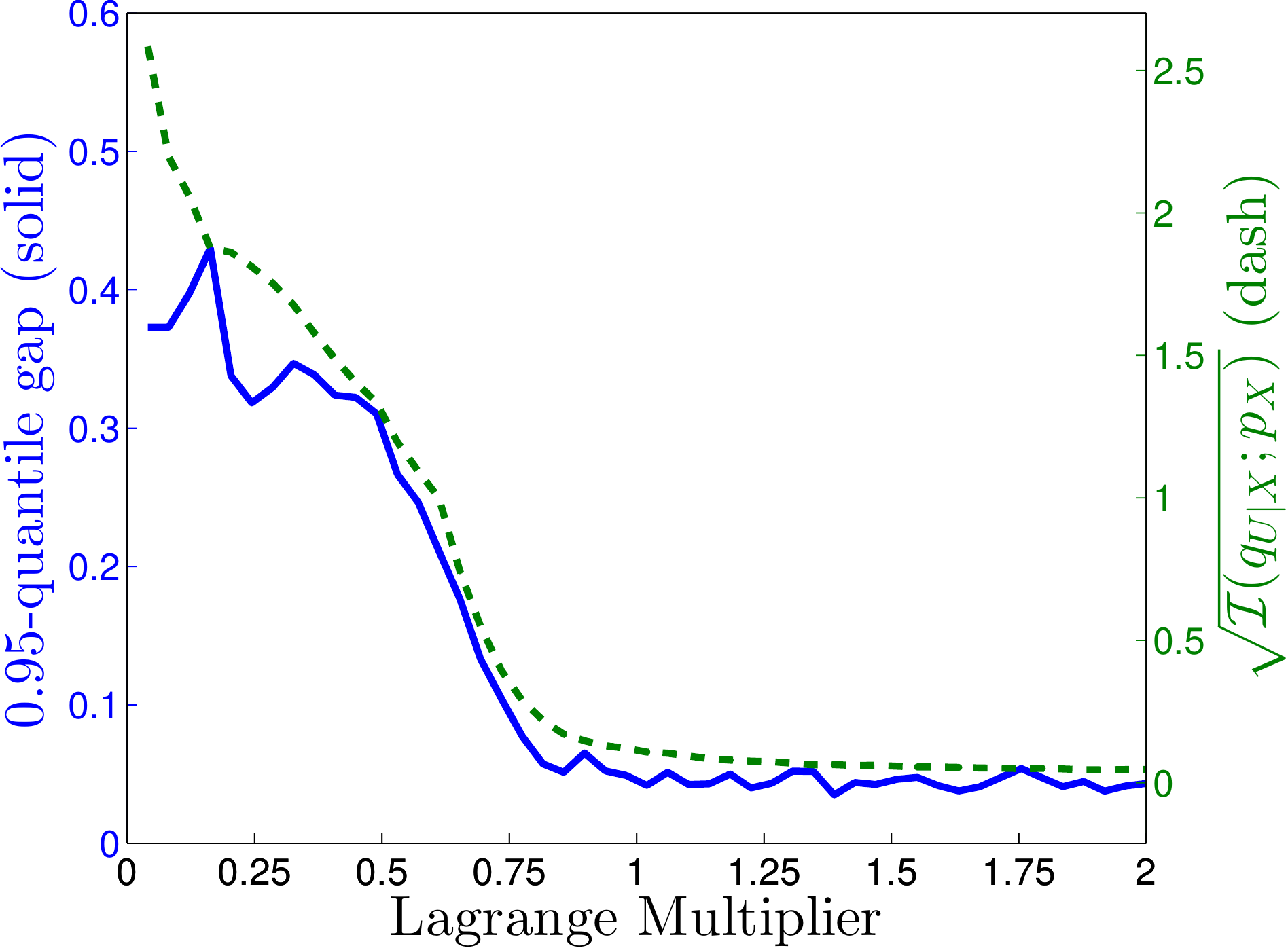}
		\caption{}
		\label{fig:lognormalmnist}
	\end{subfigure}
	\begin{subfigure}{.49\textwidth}
		\centering
		\includegraphics[scale=0.35]{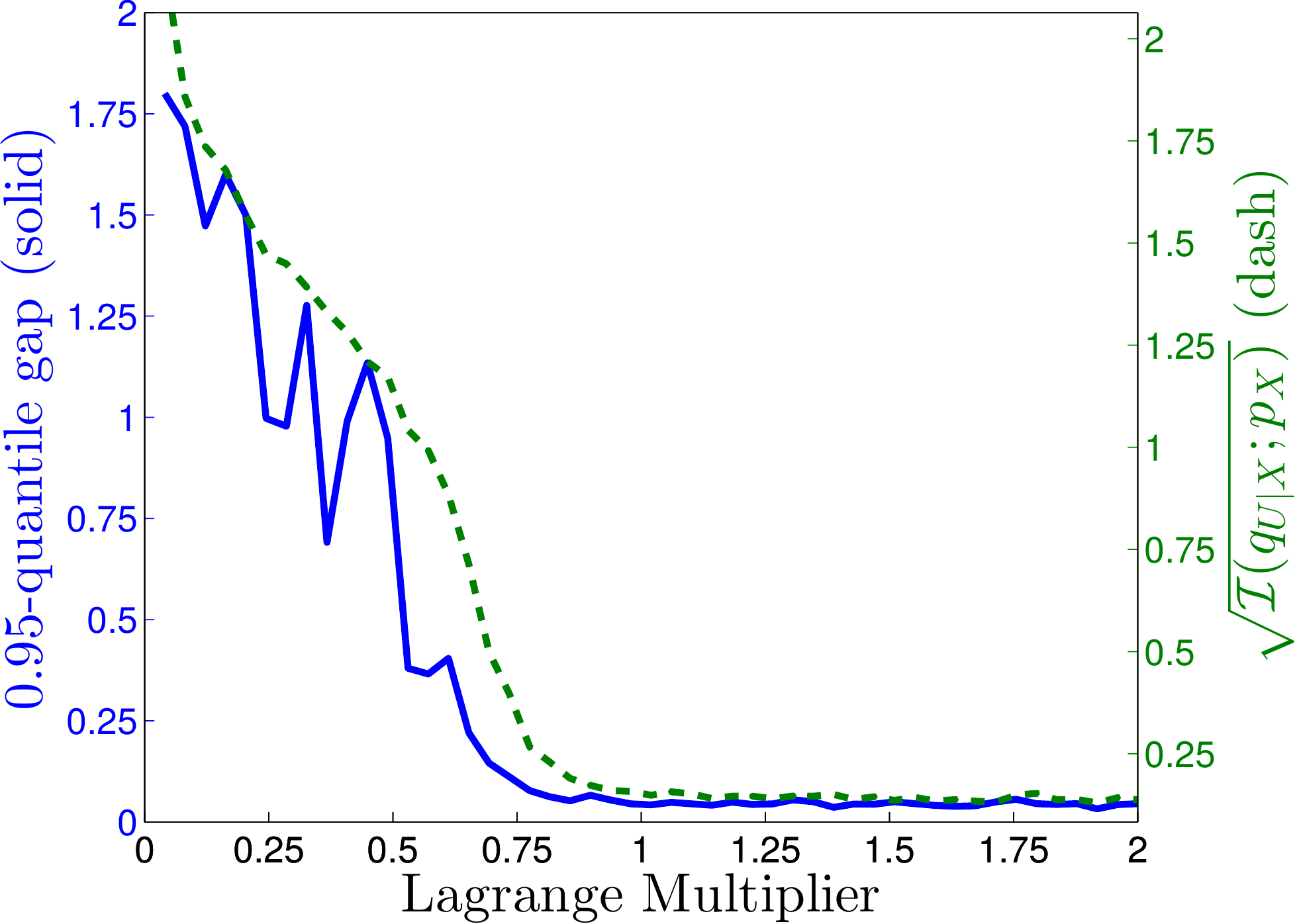}
		\caption{}
		\label{fig:lognormalnew}
	\end{subfigure}
	\caption{Comparison between $0.95$-quantile of $\mathcal{E}_{\textrm{gap}}(q_{U|X},Q_{\hat{Y}|U},\mathcal{S}_n)$ and the mutual information variational bound \eqref{eq:inf_radius_opt} for Log-Normal encoder and testing with: (a) Images generated with the training distribution,  (b) Images generated with other distribution.}
\end{figure}

\subsection{RBM Encoder: Classification using Restricted Boltzmann Machines}\label{sec:rbm_simulation}

Consider the standard models for the RBM studied in \cite{Hinton_guia,srivastava2014dropout}. For every $j\in[1:m]$, $U_j$ given $\mathbf{X}=\mathbf{x}$ is distributed as a \emph{Bernoulli} RV with parameter $\sigma(b_j+\mathbf{w}_j^T\mathbf{x})$ (sigmoid activation). Selecting the product distribution $
\tilde{q}_{U_j}(u_j)=\frac{1}{n}\sum_{i=1}^nq_{U_j|\mathbf{X}}(u_j|\mathbf{x}_i)
$, we obtain
\begin{align}
\mathcal{D}&\left(q_{U_j|\mathbf{X}}(\cdot|\mathbf{X}) \big\| \tilde{q}_{U_j}\big| \hat{P}_{\mathbf{X}} \right)=\frac{1}{n}\sum_{i=1}^n\sigma(b_j+\langle\mathbf{w}_j,\mathbf{x}_i\rangle)\log\left(\frac{\sigma(b_j+\langle\mathbf{w}_j,\mathbf{x}_i\rangle)}{\frac{1}{n}\sum_{k=1}^n\sigma(b_j+\langle\mathbf{w}_j,\mathbf{x}_k\rangle)}\right)\nonumber\\
&\qquad\qquad+\left(1-\sigma(b_j+\langle\mathbf{w}_j,\mathbf{x}_i\rangle)\right)\log\left(\frac{1-\sigma(b_j+\langle\mathbf{w}_j,\mathbf{x}_i\rangle)}{\frac{1}{n}\sum_{k=1}^n1-\sigma(b_j+\langle\mathbf{w}_j,\mathbf{x}_k\rangle)}\right).\label{eq:rbm}
\end{align}
\begin{figure}
	\centering
	\begin{subfigure}{.49\textwidth}
		\centering
		\includegraphics[scale=0.35]{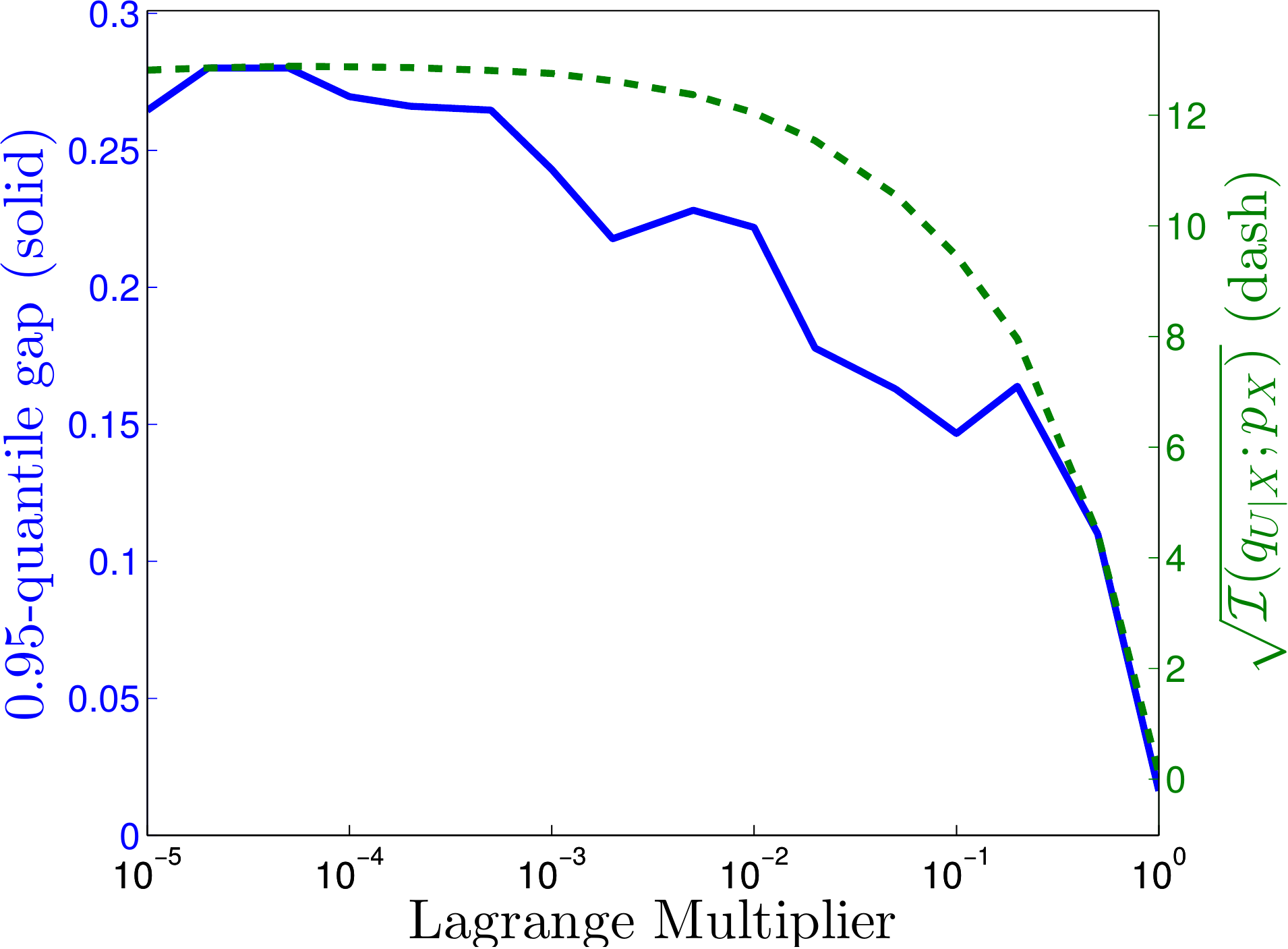}
		\caption{}
		\label{fig:rbmmnist}
	\end{subfigure}
	\begin{subfigure}{.49\textwidth}
		\centering
		\includegraphics[scale=0.35]{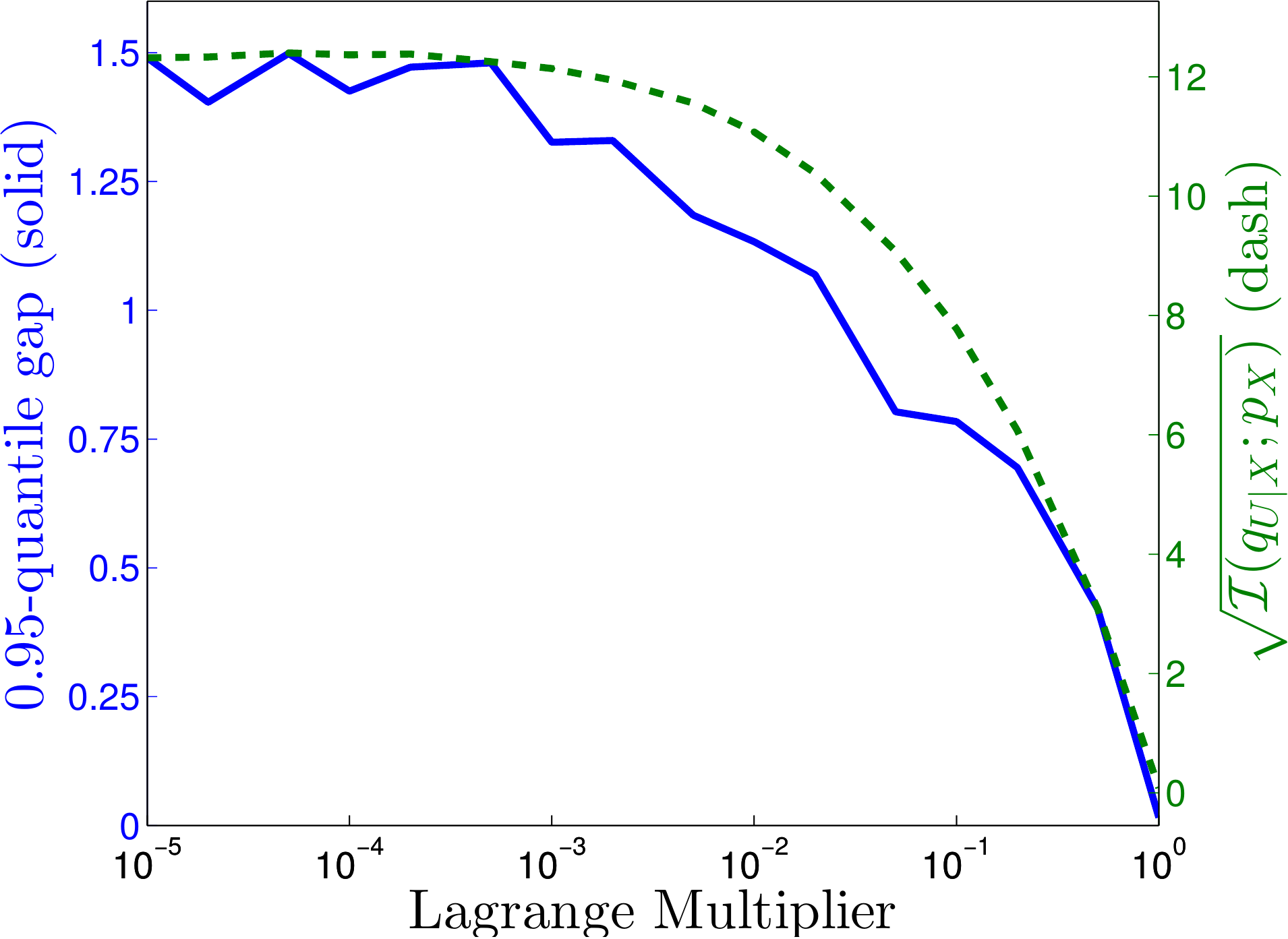}
		\caption{}
		\label{fig:rbmnew}
	\end{subfigure}
	\caption{Comparison between $0.95$-quantile of $\mathcal{E}_{\textrm{gap}}(q_{U|X},Q_{\hat{Y}|U},\mathcal{S}_n)$ and the mutual information variational bound \eqref{eq:inf_radius_opt} for RBM encoder and testing with: (a) Images generated with the training distribution,  (b) Images generated with other distribution.}
	\label{fig:rbm}
\end{figure}

Figures \ref{fig:rbmmnist} and \ref{fig:rbmnew} show $0.95$-quantile  of $\mathcal{E}_{\textrm{gap}}(q_{U|X},Q_{\hat{Y}|U},\mathcal{S}_n)$ and the mutual information variational bound \eqref{eq:inf_radius_opt} as a function of the Lagrange multiplier $\lambda$, testing with the training and disturbed MNIST dataset respectively. Again, the behaviors are similar and decreasing, both when testing with samples generated with the training distribution and with samples generated with another one.

\section{Summary and Concluding Remarks}
\label{sec:conclusion}

We presented a theoretical investigation of a typical classification task in which we have training data from a source domain, but we wish the testing gap between the empirical cross-entropy and its statistical expectation (measured with respect to a possible different testing probability law) to be as small as possible.  Our main result (Theorem~\ref{thm:regularizar}) is that the testing gap  can be bounded with high probability by the mutual information between the input testing samples and the corresponding representations, the Hellinger distance which measures the decoder efficiency and other less relevant constants. Empirical study of this metric suggests that the mutual information may be a good measure to capture the dynamic of the gap with respect to important training parameters. We finally presented a simple experimental setup which shows that there is a strong correlation between the gap represented and the mutual information between the raw inputs and their representations. Some further work will be needed to provide strong support to these numerical results in presence of other sources of non-stationarities between training and testing datasets.




\section*{Acknowledgment}
   
The work of Prof. Pablo Piantanida was supported by the European Commission’s Marie Sklodowska-Curie Actions (MSCA), through the Marie Sklodowska-Curie IF (H2020-MSCAIF-2017-EF-797805). The work of Matias Vera was supported by Peruilh PhD Scholarship from Facultad de Ingenier\'ia, Universidad de Buenos Aires.

\bibliographystyle{apalike}
\bibliography{mybibfile}

\newpage

\begin{figure*}[ht]
	\centering
	{\LARGE\textbf{Appendix}}
\end{figure*}

\appendix

\subsubsection*{Notation and conventions}
  
Let $\mathcal{P}(\mathcal{X})$ denote the set of all probability measures $\mathbb{P}$ over the set $\mathcal{X}$. The probability mass (pmf) function or probability density function (pdf) in the case of continuous random variables of $X$ is denoted interchangeably by $P_X$ or $p_X$. $\text{Supp}(p_X)$ denotes the support of the distribution, i.e. the closure of $\{x\in\mathcal{X}: p_X(x)>0\}$. $\text{Vol}\left(\mathcal{X}\right)=\int_{\mathcal{X}}dx\;$ if $\;\mathcal{X}$ is continuous or $\text{Vol}\left(\mathcal{X}\right)=|\mathcal{X}|\;$ if $\;\mathcal{X}$ is discrete. $\|\cdot\|_2$ denotes the usual Euclidean norm of a vector and $\langle\cdot,\cdot\rangle$ the canonical inner product. We use $\mathbb{E}_p[\cdot]$ and $\text{Var}_p(\cdot)$ to denote the mathematical expectation and variance respectively, measured with respect to $p$ . The information measures to be used in the work are~\cite{cover}: the \emph{entropy}  $ \mathcal{H}(P_X)\coloneqq \mathbb{E}_{P_X}\left[-\log { {P}_X}(X)\right]$; the \emph{conditional  entropy} $ \mathcal{H}(P_{Y|X}|P_X)\coloneqq \mathbb{E}_{P_{X}P_{Y|X}}\left[-\log {P}_{Y|X} (Y|X)\right]$; the \emph{relative entropy}: 
$\mathcal{D}( {P}_{X}\| {Q}_{X})\coloneqq\mathbb{E}_{P_{X}}\left[\log \frac{P_{X}(X)}{Q_X(X)}\right]$; the conditional  \emph{relative entropy}: 
$
\mathcal{D}( {P}_{Y|X}\| {Q}_{Y|X}| {P}_{X})\coloneqq \mathbb{E}_{ {P_X}}\left[ \mathcal{D}\big( {P}_{Y|X}(\cdot| X)\| {Q}_{Y|X}(\cdot| X)\big) \right] 
$
and the \emph{mutual information}: $\mathcal{I}(P_X;{P}_{Y|X})\coloneqq\mathcal{D}( {P}_{Y|X} \| {{P}_Y} | P_X)$. We talk about testing dataset $\mathcal{S}_{n}$ of $n$ samples. We study a deviation bound, after training (over $\mathcal{S}_{n}$), so we assume that the training set is given implicitly.

\section{Proof of Theorem \ref{thm:regularizar}}

In this Appendix we will prove the main result of this work. We will use some well-known results, listed in \ref{app:auxiliary-results}. We are looking for a relationship between error gap and mutual information, which is a \emph{information-theoretic} measure. Information theory in general and mutual information in particular has several known results, a large percentage of them being for discrete spaces. For example, \citet{Shamir:2010:LGI:1808343.1808503} bound the deviation over the mutual information between labels $Y$ and hidden representations $U$, through the mutual information between hidden representations and inputs $X$ with discrete alphabets. However, in most of learning problems it is more appropriate to consider a continuous alphabet $\mathcal{X}$. In order to use those discrete results, our first step is to analyze the error introduced due to a reasonable discretization. This approach has points in common with \citet{Xu2012} \emph{robust-algorithms theory}.

\begin{lemma}\label{lem:paso1disc}
	Let $|\mathcal{Y}|$ different partitions of $\mathcal{X}$ and their respective centroids for each $y\in\mathcal{Y}$ $\left(\{\mathcal{K}_k^{(y)}\}_{k=1}^K,\{x^{(k,y)}\}_{k=1}^K\right)$ function of the natural number $K$, where partitions meet for each $y\in\mathcal{Y}$:  $\bigcup_{k=1}^K\mathcal{K}_k^{(y)}=\mathcal{X},\;\mathcal{K}_{i}^{(y)}\cap\mathcal{K}_{j}^{(y)}=\emptyset\;\forall1\leq i<j\leq K,\;\int_{\mathcal{K}_k^{(y)}}dx>0\;\forall 1\leq k\leq K$, the error gap \eqref{def-gap} can be bounded as $\mathcal{E}_{\textrm{gap}}(q_{U|X},Q_{\hat{Y}|U},\mathcal{S}_n)\leq2\epsilon(K)+\mathcal{E}_{\textrm{gap}}^D(q_{U|X},Q_{\hat{Y}|U},\mathcal{S}_n)$ almost surely, where $\epsilon(K)$ was defined in \eqref{eq:epsilon_r_definitions} and $\mathcal{E}_{\textrm{gap}}^D(q_{U|X},Q_{\hat{Y}|U},\mathcal{S}_n)$ is defined as
	\begin{equation}
	\mathcal{E}_{\textrm{gap}}^D(q_{U|X},Q_{\hat{Y}|U},\mathcal{S}_n)=\left|\mathcal{L}^D(q_{U|X},Q_{\hat{Y}|U})-\mathcal{L}_{\textrm{emp}}^D(q_{U|X},Q_{\hat{Y}|U},\mathcal{S}_n)\right|
	\end{equation}
	where
	\begin{align}
	\mathcal{L}^D(q_{U|X},Q_{\hat{Y}|U})&=\sum_{k=1}^K\sum_{y\in\mathcal{Y}}P^D_{XY}(k,y)\ell(x^{(k,y)},y),\label{eq:ld}\\
	\mathcal{L}_{\textrm{emp}}^D(q_{U|X},Q_{\hat{Y}|U},\mathcal{S}_n)&=\frac{1}{n}\sum_{k=1}^{K}\sum_{\substack{i\in[1:n]\\x_i\in\mathcal{K}_k}}\ell(x^{(k,y_i)},y_i)\label{eq:ldemp}
	\end{align} 
	and $P^D_{XY}:\;\{x^{(k,y)}:\;k\in[1:K], y\in\mathcal{Y}\}\times\mathcal{Y}\rightarrow[0,1]$ is
	\begin{equation}
	P^D_{XY}(x,y)=\sum_{k=1}^K\mathds{1}\left\{x=x^{(k,y)}\right\}\int_{\mathcal{K}_k^{(y)}}p_{XY}(x^\prime,y)dx^{\prime}.
	\end{equation} 
\end{lemma}
On the one hand, $P^D_{XY}$ is a pmf such that $P_Y$ (true value) is its marginal. On the other hand the other marginal $P_X^D(x)=\sum_{y\in\mathcal{Y}}P^D_{XY}(x,y)$ has the elements of the set $\mathcal{A}=\{x^{(k,y)}:\;1\leq k\leq K,\;y\in\mathcal{Y}\}$ as atoms. 
\begin{proof}
	Triangle inequality allow to relate error gaps as,
	\begin{align}
	\mathcal{E}_{\textrm{gap}}(q_{U|X},Q_{\hat{Y}|U},\mathcal{S}_n)& \leq\left|\mathcal{L}^D(q_{U|X},Q_{\hat{Y}|U})-\mathcal{L}_{\textrm{emp}}^D(q_{U|X},Q_{\hat{Y}|U},\mathcal{S}_n)\right|\nonumber\\
	&+\left|\mathcal{L}(q_{U|X},Q_{\hat{Y}|U})-\mathcal{L}^D(q_{U|X},Q_{\hat{Y}|U})\right|\nonumber\\
	&+\left|\mathcal{L}_{\textrm{emp}}^D(q_{U|X},Q_{\hat{Y}|U},\mathcal{S}_n)-\mathcal{L}_{\textrm{emp}}(q_{U|X},Q_{\hat{Y}|U},\mathcal{S}_n)\right|, \label{eq:discretizandoelgap}
	\end{align}
	where the last term is the discrete error gap $\mathcal{E}_{\textrm{gap}}^D(q_{U|X},Q_{\hat{Y}|U},\mathcal{S}_n)$. The other terms in \eqref{eq:discretizandoelgap} can be bounded using the fact that $P_X^D(x^{(k,y)})=\mathbb{P}\left(X\in\mathcal{K}_k^{(Y)},Y=y\right)$ and definition of $\epsilon(K)$ \eqref{eq:epsilon_r_definitions}:
	\begin{align}
	&\left|\mathcal{L}(q_{U|X},Q_{\hat{Y}|U})-\mathcal{L}^D(q_{U|X},Q_{\hat{Y}|U})\right|\nonumber\\
	&\;=\left|\sum_{k=1}^K\sum_{y=1}^{|\mathcal{Y}|}P^D_{XY}(x^{(k,y)},y)\left(\mathbb{E}\left[\ell(X,Y)|Y=y,X\in\mathcal{K}_k^{(Y)}\right]-\ell(x^{(k,y)},y)\right)\right|\\
	&\;\leq\epsilon(K)\\
	&\left|\mathcal{L}_{\textrm{emp}}^D(q_{U|X},Q_{\hat{Y}|U},\mathcal{S}_n)-\mathcal{L}_{\textrm{emp}}(q_{U|X},Q_{\hat{Y}|U},\mathcal{S}_n)\right|\nonumber\\
	&\;=\frac{1}{n}\left|\sum_{k=1}^K\sum_{\substack{i\in[1:n]\\x_i\in\mathcal{K}_k}}\left[\ell(x^{(k,y_i)},y_i)-\ell(x_i,y_i)\right]\right|\\
	&\;\leq\epsilon(K).
	\end{align}
\end{proof}

The second step in our proof is to measure the decoupling between the encoder and decoder, i.e. what is the error when considering $Q^D_{Y|U}$ \eqref{eq:Qy|u_D} as decoder. The following lemma separates the decoder term.

\begin{lemma}\label{lem:paso2dec}
	Discrete error gap can be bounded as
	\begin{equation}
	\mathcal{E}_{\textrm{gap}}^D(q_{U|X},Q_{\hat{Y}|U},\mathcal{S}_n)\leq\mathcal{E}_{\textrm{gap}}^D(q_{U|X},Q^D_{Y|U},\mathcal{S}_n)+d(\mathcal{S}_n)\label{eq:enc-dec}
	\end{equation}
	almost surely, where
	\begin{equation}
	d(\mathcal{S}_n)=\left|\frac{1}{n}\sum_{k=1}^K\sum_{\substack{i\in[1:n]\\x_i\in\mathcal{K}_k}}T(x^{(k,y_i)},y_i)-\mathcal{D}\left(Q^D_{Y|U}\|Q_{\hat{Y}|U}|q^D_U\right)\right|
	\end{equation}
	and $T(x,y)\coloneqq\mathbb{E}_{q_{U|X}}\left[\left.\log\left(\frac{Q^D_{Y|U}(y|U)}{Q_{\hat{Y}|U}(y|U)}\right)\right|X=x\right]$.
\end{lemma}
Note that $\mathbb{E}_{P^D_{XY}}\left[T(X,Y)\right]=\mathcal{D}\left(Q^D_{Y|U}\|Q_{\hat{Y}|U}|q^D_U\right)$, so $d(\mathcal{S}_n)$ is a deviation of $T(X,Y)$.
\begin{proof}
	It is immediate to see that:
	\begin{equation}
	\ell\big(q_{U|X}(\cdot|x),{Q}_{\hat{Y}|U}(y | \cdot)\big)=\ell\big(q_{U|X}(\cdot|x),{Q}^D_{Y|U}(y | \cdot)\big)+T(x,y)
	\end{equation}
	So, taking expectation (as \eqref{eq:ld} and \eqref{eq:ldemp}) and using triangle inequality, we can prove the lemma.
\end{proof}

The third step of our proof is to bound the new gap $\mathcal{E}_{\textrm{gap}}^D(q_{U|X},Q^D_{Y|U},\mathcal{S}_n)$. For that, we use empirical distributions $\hat{P}^D_{XY},\hat{P}^D_{X}, \hat{P}_Y, \hat{P}^D_{X|Y}$ as the occurrence rate of $\mathcal{S}_n$; e.g. $\hat{P}^D_{XY}(k,y)=\frac{|\left\{(x_i,y_i)\in\mathcal{S}_n:\; y_i=y,\;x_i\in\mathcal{K}_k^{(y)}\right\}|}{n}$. Also, we define the empirical distributions generated from the encoder $\hat{q}^D_U, \hat{Q}^D_{Y|U}, \hat{q}^D_{U|Y}$ generated from the encoder; e.g. $q_U^D(u)=\sum_{x\in\mathcal{A}}q_{U|X}(u|x)\hat{P}^D_{X}(x)$.
\begin{rem}
	The marginal $P_Y$ and its empirical pmf $\hat{P}_Y$ has not got superscript $D$ because it matches with the true pmf (without quantization).
\end{rem}

\begin{lemma}\label{lem:paso3cotaas}
	The gap $\mathcal{E}_{\textrm{gap}}^D(q_{U|X},Q^D_{Y|U},\mathcal{S}_n)$ can be bounded as, 
	\begin{align}
	\mathcal{E}_{\textrm{gap}}^D&(q_{U|X},Q^D_{Y|U},\mathcal{S}_n)\leq \mathcal{D}\left(\hat{P}_{XY}^D\big \|P_{XY}^D\right)+\int_{\mathcal{U}}\phi\left(\left\|\mathbf{P}_{X}^D\!-\!\mathbf{\hat{P}}_{X}^D\right\|_2\!\sqrt{\mathbb{V}\left(\mathbf{q}_{U|X}(u|\cdot) \right)}\right)du\nonumber\\
	&+\log \left(\frac{\text{Vol}\left(\mathcal{U}\right)}{P_{Y}(y_{\min})}\right)\sqrt{|\mathcal{Y}|}\left\|\mathbf{P}_{Y}\!-\!\mathbf{\hat{P}}_{Y}\right\|_2+\!\mathcal{O}\left(\|\mathbf{P}_{Y}\!-\!\mathbf{\hat{P}}_{Y}\|_2^2\right)\!\nonumber\\
	&+\mathbb{E}_{P_Y}\Big[ \int_\mathcal{U}\!\phi\left(\left\|\mathbf{P}_{X|Y}^D(\cdot|Y)\!-\!\mathbf{\hat{P}}_{X|Y}^D(\cdot|Y)\right\|_2  \!\sqrt{\mathbb{V}\left(\mathbf{q}_{U|X}(u|\cdot)\right)}\right)du\Big]\label{eq:todoslosterminos}
	\end{align}
	almost surely, where $\phi(\cdot)$ is defined in \eqref{eq:cotaphi} and 
	\begin{equation}
	\mathbb{V}(\mathbf{c})\coloneqq \left\|\mathbf{c}-\bar{c}\mathds{1}_a\right\|_2^2,
	\label{eq:emp_var}
	\end{equation}
	with $\mathbf{c}\in\mathbb{R}^a$, $a\in\mathbb{N}$, $\bar{c}=\frac{1}{a}\sum_{i=1}^a c_i$, and $\mathds{1}_a$ is the vector of ones of length $a$.
\end{lemma}
Notation $\mathbf{P}_Y$ means the pmf $P_Y$ as a vector $\mathbf{P}_Y=[P_Y(1),\cdots,P_Y(|\mathcal{Y}|)]$, so we can apply it vector norms $\|\cdot\|_2$ and $\mathbb{V}(\cdot)$ operator. This operator measures the dispersion of the components of $\bar{a}$ around the mean, where  $\mathbb{V}(\mathbf{c}) \leq\left\|\mathbf{c}-b\mathds{1}_a\right\|_2^2\ , \ \forall b\in\mathbb{R}$.

\begin{proof}
	Adding and subtracting $\hat{P}_{XY}^D(x^{(k,y)},y)\mathbb{E}_{q_{U|X}}\left[\left.\log\left(\frac{1}{\hat{Q}^D_{Y|U}(y|U)}\right)\right|X=x^{(k,y)}\right]$ we can prove via triangle inequality that
	\begin{align}
	&\mathcal{E}_{\textrm{gap}}^D(q_{U|X},Q^D_{Y|U},\mathcal{S}_n)\\
	&=\left|\sum_{\forall (k,y)}\left[P_{XY}^D(x^{(k,y)},y)-\hat{P}_{XY}^D(x^{(k,y)},y)\right]\mathbb{E}_{q_{U|X}}\left[\left.\log\left(\frac{1}{Q^D_{Y|U}(y|U)}\right)\right|X=x^{(k,y)}\right]\right|\nonumber\\
	&\leq\left| \mathcal{H}(Q^D_{Y|U}|q^D_U) - \mathcal{H}(\hat{Q}^D_{Y|U}|\hat{q}^D_U)\right|+ \mathcal{D}\left(\hat{Q}^D_{Y|U}\big\|Q^D_{Y|U}\big| \hat{q}^D_U \right).\label{eq:segundacota}
	\end{align}
	We can bound the second term in \eqref{eq:segundacota} using the inequality: 
	\begin{equation}
	\mathcal{D}\left(\hat{Q}^D_{Y|U}\big\|Q^D_{Y|U}\big| \hat{q}^D_U \right)\leq\mathcal{D}\left(\hat{Q}^D_{Y|U}\hat{q}^D_U\big\|Q^D_{Y|U}q^D_U\right)\leq
	\mathcal{D}\left(\hat{P}_{XY}^D\big\|P_{XY}^D\right).
	\end{equation}
	The first term of \eqref{eq:segundacota} can be bounded as:
	\begin{align}
	\left| \mathcal{H}(Q^D_{Y|U}|q^D_U)\!-\! \mathcal{H}(\hat{Q}^D_{Y|U}|\hat{q}^D_U)\right| &\leq\left| \mathcal{H}(P_Y)- \mathcal{H}(\hat{P}_{Y})\right|+\left| \mathcal{H}_d(q^D_{U})\!-\! \mathcal{H}_d(\hat{q}^D_{U})\right|\nonumber\\
	&\qquad+\left| \mathcal{H}_d(q^D_{U|Y}|P_Y)- \mathcal{H}_d(\hat{q}^D_{U|Y}|\hat{P}_Y)\right|,
	\end{align}
	where $\mathcal{H}_d$ is the differential entropy when $U$ is continuous and the classical entropy when $U$ is discrete. The terms $\left| \mathcal{H}_d(q^D_{U})\!-\! \mathcal{H}_d(\hat{q}^D_{U})\right|$ and 
	$\left| \mathcal{H}_d(q^D_{U|Y}|P_Y)- \mathcal{H}_d(\hat{q}^D_{U|Y}|\hat{P}_Y)\right|$ can be bounded by Lemmas~\ref{lem:Hu} and~\ref{lem:Hu|y} respectively. Finally, it is clear that
	$P_Y\mapsto  \mathcal{H}(P_Y)$ is differentiable and a first order Taylor expansion yields:
	\begin{equation}
	\mathcal{H}(P_Y)- \mathcal{H}(\hat{P}_Y) =\left\langle\frac{\partial  \mathcal{H}(P_Y)}{\partial \mathbf{P}_{Y}},\mathbf{P}_{Y}\!-\!\mathbf{\hat{P}}_{Y}\!\!\right\rangle+\mathcal{O}\left(\|\mathbf{P}_{Y}\!-\!\mathbf{\hat{P}}_{Y}\|_2^2\right),
	\end{equation}
	where
	$\frac{\partial  \mathcal{H}(P_Y)}{\partial P_{Y}(y)}=-\log P_Y(y)-1$ for each $y\in\mathcal{Y}$. Then, applying Cauchy-Schwartz inequality the lemma was proved: 
	\begin{align}
	\left| \mathcal{H}(P_Y)\!-\! \mathcal{H}(\hat{P}_Y)\right| &\leq\left|\left\langle  \log \mathbf{P}_Y,\mathbf{P}_{Y}\!-\!\mathbf{\hat{P}}_{Y}\right\rangle\right|\!+\! \mathcal{O}\left(\!\|\mathbf{P}_{Y}\!-\!\mathbf{\hat{P}}_{Y}\|_2^2\!\right)\\
	&\leq\left\| \log \mathbf{P}_Y\right\|_2\left\|\mathbf{P}_{Y}\!-\!\mathbf{\hat{P}}_{Y}\right\|_2\!+\!\mathcal{O}\left(\|\mathbf{P}_{Y}\!-\!\mathbf{\hat{P}}_{Y}\|_2^2\right)\\
	&\leq \log\left(\frac1{P_{Y}(y_{\min})}\right)\sqrt{|\mathcal{Y}|}\left\|\mathbf{P}_{Y}\!-\!\mathbf{\hat{P}}_{Y}\right\|_2\!+\!\mathcal{O}\left(\|\mathbf{P}_{Y}\!-\!\mathbf{\hat{P}}_{Y}\|_2^2\right).
	\end{align}
\end{proof}

The combination of lemmas \ref{lem:paso1disc}, \ref{lem:paso2dec} and \ref{lem:paso3cotaas} allow to bound the error gap with probability one. The fourth step is to use a concentration inequality over the terms: $\mathcal{D}\big(\hat{P}_{XY}^D\|P_{XY}^D\big)$, $\|\mathbf{P}_{X}^D-\mathbf{\hat{P}}_{X}^D\|_2$, $\|\mathbf{P}_{Y}-\mathbf{\hat{P}}_{Y}\|_2$, $\| \mathbf{P}_{X|Y}^D(\cdot|y)-\mathbf{\hat{P}}_{X|Y}^D(\cdot|y)\|_2$ for $y\in\mathcal{Y}$ and $d(\mathcal{S}_n)$ simultaneously. Lemma \ref{lem:simultaneidad} guarantees  that the bounds hold simultaneously over all these $|\mathcal{Y}|+4$ quantities, by replacing $\delta$ with $\delta/(|\mathcal{Y}|+4)$. With probability at least $1-\delta$, we apply Lemmas \ref{lem:diverg}, \ref{lem:pvector} and Chebyshev inequality \cite{Devroye97a}:
\begin{align}
&\mathcal{D}\left(\hat{P}_{XY}^D\|P_{XY}^D\right)\leq |\mathcal{X}||\mathcal{Y}|\frac{\log(n+1)}{n}+\frac{1}{n}\log\left(\frac{|\mathcal{Y}|+4}{\delta}\right)=\mathcal{O}\left(\frac{\log(n)}{n}\right).\label{eq:concentracionkl}\\
&\max\Big\{\big\|\mathbf{P}_{Y}-\mathbf{\hat{P}}_{Y}\big\|_2,\big\|\mathbf{P}_{X}^D-\mathbf{\hat{P}}_{X}^D\big\|_2,\big\|\mathbf{P}_{X|Y}^D(\cdot|y)-\mathbf{\hat{P}}^D_{X|Y}(\cdot|y)\big\|_2\Big\}\nonumber\\ &\hspace{30mm}\leq\frac{1+\sqrt{\log\left(\frac{|\mathcal{Y}|+4}{\delta}\right)}}{\sqrt{n}}\equiv\frac{B_\delta}{\sqrt{n}},\label{eq:concentracionpv}\\
&\hspace{21mm}d(\mathcal{S}_n)\leq\sqrt{\frac{|\mathcal{Y}|+4}{n\delta}}\sqrt{\text{Var}_{P_{XY}^D}\left(T(X,Y)\right)}.\label{eq:concentracionch}
\end{align}
In order to bound the last variance we enunciate the following lemma.
\begin{lemma}\label{lem:ghosal}
	Variance of $T$ random variable can be bounded as
	\begin{align}
	\text{Var}_{P_{XY}^D}\left(T(X,Y)\right)&\leq\frac{8}{\sqrt{Q_{\hat{Y}|U}(y_{\min}|u_{\min})}}\mathcal{D}_{\text{HL}}^2\left(Q^D_{Y|U}\|Q_{\hat{Y}|U}|q^D_U\right),
	\end{align}
	where $h$ is the Hellinger distance \eqref{eq:hellinger}.
\end{lemma}
\begin{proof}
	A similar result can be founded \cite{Ghosal2000}. Variance can be bounded as
	\begin{align}
	\text{Var}_{P_{XY}^D}\left(T(X,Y)\right)&=\text{Var}_{P_{XY}^D}\left(\mathbb{E}_{q_{U|X}}\left[\left.\log\left(\frac{Q^D_{Y|U}(Y|U)}{Q_{\hat{Y}|U}(Y|U)}\right)\right|X\right]\right)\\
	&\leq\text{Var}_{Q^D_{Y|U}q^D_{U}}\left(\log\left(\frac{Q^D_{Y|U}(Y|U)}{Q_{\hat{Y}|U}(Y|U)}\right)\right)\\
	&\leq\mathbb{E}_{Q^D_{Y|U}q^D_{U}}\left[\log^2\left(\frac{Q^D_{Y|U}(Y|U)}{Q_{\hat{Y}|U}(Y|U)}\right)\right].
	\end{align}
	Note that $\frac{1}{Q_{\hat{Y}|U}(y_{\min}|u_{\min})}\geq\sup\left\{\frac{Q^D_{Y|U}(y|u)}{Q_{\hat{Y}|U}(y|u)},1\right\}$. For every $c\leq0$ and $x\geq c$ we have the inequality $x^2\leq 4e^{-c/2}\left(e^{x/2}-1\right)^2$. So call $c=\log Q_{\hat{Y}|U}(y_{\min}|u_{\min})$ and $x=\log\frac{Q_{\hat{Y}|U}(y|u)}{Q^D_{Y|U}(y|u)}$, the inequality can be written as:
	\begin{equation}
	\log^2\frac{Q^D_{Y|U}(y|u)}{Q_{\hat{Y}|U}(y|u)}\leq\frac{4}{\sqrt{Q_{\hat{Y}|U}(y_{\min}|u_{\min})}}\left(\sqrt{
		\frac{Q_{\hat{Y}|U}(y|u)}{Q^D_{Y|U}(y|u)}}-1\right)^2.
	\end{equation}
	Taking expectation term by term over $q_U^D Q^D_{Y|U}$ the proof is over. 
\end{proof}

We generate the following bound for the error gap using  concentration inequalities \eqref{eq:concentracionkl}, \eqref{eq:concentracionpv} and \eqref{eq:concentracionch}.

\begin{lemma}\label{lem:coninfodiscreta}
	For every $\delta\in(0,1)$, with probability at least $1-\delta$ over the choice of $\mathcal{S}_n\sim p_{XY}$,  the gap satisfies:
	\begin{align}
	\mathcal{E}_{\textrm{gap}}(q_{U|X},Q_{\hat{Y}|U},\mathcal{S}_n)&\leq\inf_{K\in\mathbb{N}}2\epsilon(K)+  A_\delta\sqrt{\mathcal{I}(P^D_X;q_{U|X})}\cdot\frac{\log(n)}{\sqrt{n}}r(K)\nonumber\\
	&\hspace{-0.1cm}+\frac{D_\delta\cdot \mathcal{D}_{\text{HL}}\left(Q^D_{Y|U}\|Q_{\hat{Y}|U}|q^D_U\right)+C_\delta} {\sqrt{n}}+\mathcal{O}\left(\frac{\log(n)}{n}\right),
	\end{align} 
	$\forall\,(q_{U|X},Q_{\hat{Y}|U})$ that meets Assumptions \ref{asumption1}.
\end{lemma}
\begin{proof}
	We bound the error gap using lemmas \ref{lem:paso1disc}, \ref{lem:paso2dec}, \ref{lem:paso3cotaas}, \ref{lem:ghosal} and \ref{lem:cotaphi}, with probability at least $1-\delta$:
	\begin{align}
	\mathcal{E}_{\textrm{gap}}(q_{U|X},Q_{\hat{Y}|U},\mathcal{S}_n)
	&\leq2\epsilon(K)+ \frac{D_\delta}{\sqrt{n}}\mathcal{D}_{\text{HL}}\big({Q}^D_{Y|U}\|Q_{\hat{Y}|U}|q^D_U\big)+\log \left(\frac{\text{Vol}\left(\mathcal{U}\right)}{P_{Y}(y_{\min})}\right)\sqrt{|\mathcal{Y}|}\frac{B_\delta}{\sqrt{n}}\nonumber\\
	&\quad+2\int_{\mathcal{U}}\phi\left(\frac{B_\delta}{\sqrt{n}}\sqrt{\mathbb{V}\left( \textbf{q}_{U|X}(u|\cdot ) \right)} \right)du+\mathcal{O}\left(\frac{\log(n)}{n}\right)\\
	&\leq2\epsilon(K)+ \frac{D_\delta}{\sqrt{n}}\mathcal{D}_{\text{HL}}\big({Q}^D_{Y|U}\|Q_{\hat{Y}|U}|q^D_U\big)+\log \left(\frac{\text{Vol}\left(\mathcal{U}\right)}{P_{Y}(y_{\min})}\right)\sqrt{|\mathcal{Y}|}\frac{B_\delta}{\sqrt{n}}\nonumber\\
	&\quad+\frac{\log(n)}{\sqrt{n}}B_\delta\int_{\mathcal{U}}\sqrt{\mathbb{V}\left( \mathbf{q}_{U|X}(u|\cdot)  \right)}du +\frac{2\text{Vol}(\mathcal{U})e^{-1}}{\sqrt{n}}+\mathcal{O}\left(\frac{\log(n)}{n}\right).
	\end{align}
	We relate the mutual information $\mathcal{I}(P_X^D;q_{U|X})$ with  $\int_{\mathcal{U}}\sqrt{\mathbb{V}\left( \mathbf{q}_{U|X}(u|\cdot)  \right)}du$. Its proof follows from an application of Pinsker's inequality \cite[Lemma 11.6.1]{cover} $\|\mathbf{P}_1-\mathbf{P}_2\|_1^2\leq2\mathcal{D}(P_1\|P_2)$ and the fact that $\mathbb{V}(\mathbf{c}) \leq\left\|\mathbf{c}-b\mathds{1}_a\right\|_2^2\ , \ \forall b\in\mathbb{R}$:
	\begin{align}
	\sqrt{\mathbb{V}\left( \mathbf{q}_{U|X}(u|\cdot)  \right)}&\leq\sqrt{\sum_{x\in\mathcal{A}}\left[q_{U|X}(u|x)-q_U^D(u)\right]^2}\\
	&=q^D_U(u)\sqrt{\sum_{x\in\mathcal{A}}\left[\frac{Q^D_{X|U}(x|u)}{P^D_X(x)}-1\right]^2}\\
	&\leq q^D_U(u)\sum_{x\in\mathcal{A}}\left|\frac{Q^D_{X|U}(x|u)}{P^D_X(x)}-1\right|\\
	&= q^D_U(u)\sum_{x\in\mathcal{A}}\frac{1}{P_X^D(k)}\left|Q^D_{X|U}(x|u)-P_X^D(x)\right|\\
	&\leq \sqrt{2}r(K)\cdot q^D_U(u)\sqrt{\mathcal{D}\left(Q^D_{X|U}(\cdot|u)\|P^D_X\right)},
	\end{align}
	where $Q^D_{X|U}(k|u)=\frac{q_{U|X}(u|x^{(k)})P_X^D(k)}{q^D_U(u)}$. So, using Jensen inequality
	\begin{align}
	\int_{\mathcal{U}}\sqrt{\mathbb{V}\left( \mathbf{q}_{U|X}(u|\cdot)  \right)}du&\leq\sqrt{2}r(K)\mathbb{E}_{q_U^D}\left[\sqrt{\mathcal{D}\left(Q^D_{X|U}(\cdot|u)\|P^D_X\right)}\right]\\
	&\leq\sqrt{2}\cdot r(K)\cdot\sqrt{\mathcal{I}(P_X^D;q_{U|X})},
	\end{align}
	for all $K\in\mathbb{N}$. Finally the lemma is proved taking the infimum over $K$.
\end{proof}

Finally, using the identity $\mathcal{I}(p_X;q_{U|X})=\mathcal{I}(p_{XY};q_{U|X})$ and the fact that the $k$ where a sample $x$ belongs is a deterministic function of $x$ and $y$,  we can use the Data Processing Inequality \cite{cover} $\mathcal{I}(P_X^D;q_{U|X})\leq\mathcal{I}(p_{XY};q_{U|X})=\mathcal{I}(p_{X};q_{U|X})$. With this result the lemma \ref{lem:coninfodiscreta} becomes in our Theorem \ref{thm:regularizar}.

\begin{rem}
	It is worth mentioning the differences between our result and that presented in \cite{Shamir:2010:LGI:1808343.1808503}.  While we work with the cross-entropy gap making appear several terms in lemma \ref{lem:paso3cotaas}, their only bounded the mutual information gap:
	\begin{align}
	&|\mathcal{I}\big(q^D_U;Q^D_{Y|U}\big)-\mathcal{I}\big(\hat{q}^D_U;\hat{Q}^D_{Y|U}\big)|\nonumber\\
	&\qquad\leq\left| \mathcal{H}_d(q^D_{U})- \mathcal{H}_d(\hat{q}^D_{U})\right|+\left| \mathcal{H}_d(q^D_{U|Y}|P_Y)- \mathcal{H}_d(\hat{q}^D_{U|Y}|\hat{P}_Y)\right|.
	\end{align} 
	For this reason we had to use more concentration inequalities including \eqref{eq:concentracionkl} and Chebyshev, while their work only uses Lemma \ref{lem:pvector}. In addition, our final mutual information is expressed in terms of continuous pdf, while theirs is a function of the discrete pmf. Finally some constants were subtly reduced and we extend the result for continuous representations of $U$.
\end{rem}

\section{Experimental Details and Other Results}

In this section we mention the details of the simulations in Section \ref{sec:experimentos}, and present new simulations using the CIFAR-10 dataset \citep{cifar}. We sample two different random subsets of: MNIST (standard dataset of handwritten digits) and CIFAR-10 (natural images). The size of the training set is $5K$ for both datasets. It is important to emphasize  the main difference between the datasets: while MNIST propose a task that is simpler, CIFAR-10 is a more difficult one, especially with a small numbers of samples and without convolutional DNNs. From this observation, we would expect that in the latter case, the classifier  will not be trained enough to have an close to optimal performance. We approximate $\mathcal{L}(q_{U|X},Q_{\hat{Y}|U})$ with a $5K$ dataset and $\mathcal{L}_{\text{emp}}(q_{U|X},Q_{\hat{Y}|U},{\mathcal{S}_n})$ with different independent mini-testing datasets of $100$ samples, i.e., using the rest of the features. The $\mathcal{E}_{\textrm{gap}}(q_{U|X},Q_{\hat{Y}|U},\mathcal{S}_n)$ $0.95$-quantile ($\delta=0.05$) is computed based on the different values of the testing risk. The values reported in each simulation are the average of three simulations independent, choosing at random, in each case, different sets. Below are the architectures used in each setup:

\begin{itemize}
	\item {\bf Normal encoder:} The DNNs used was a feed-forward layer of $512$ hidden units with ReLU activation followed by another linear for each parameter ($\mu$ and $\log\sigma^2$) with $256$ hidden units. I.e., each parameter, $\mu$ and $\log\sigma^2$, are a two-layers network where the first one is common to both. We choose a learning rate in $0.001$, a batch-size of $100$ and we train during $200$ epochs. The cost function considered during the training phase was of the form,
	\begin{equation}\label{eq:normal-lognormal cost function}
	\mathcal{L}_{\text{emp}}(q_{U|X},Q_{\hat{Y}|U},\mathcal{D}_{l})+\lambda \sum_{j=1}^m\frac{1}{l}\sum_{i=1}^l\mathcal{D}\left(q_{U_j|\mathbf{X}}(\cdot|\mathbf{x}_i) \big\| \tilde{q}_{U_j}\right),
	\end{equation} 
	where $\lambda$ is the regulation Lagrange multiplier and $\mathcal{D}_{l}$ is the $l$-training dataset.
	\item {\bf LogNormal encoder:} The DNNs used for $\mathbf{f}(\mathbf{x})$ was a feed-forward structure with two layers of $256$ hidden units with a softplus activation and for $\boldsymbol{\alpha}(\mathbf{x})$ a feed-forward layer of $256$ hidden units with a sigmoid activation multiplied by $0.7$, so that the
	maximum variance of the log-normal error distribution will
	be approximately $1$ \cite{2016arXiv161101353A}. We choose a learning rate in $0.001$, a batch-size of $100$ and we train during $200$ epochs. The cost function trained was the same that the one in \eqref{eq:normal-lognormal cost function}.
	
	\item {\bf RBM encoder:} Eq. \eqref{eq:rbm} is difficult to use as a regularizer even using training with the contrastive divergence learning procedure \cite{Hinton2002Training}. Instead, we rely on the usual RBM regularization: weight-decay. This is a traditional way to improve the generalization capacity. We explore the effect of the Lagrange multiplier $\lambda$, so called weight-cost, over both the gap and the mutual information. This meta-parameter controls the gradient weight decay, i.e., the cost function can be written as:
	\begin{equation}
	CD_{\text{RBM}}+\frac{\lambda}{2}\|\mathbf{W}\|_F^2,
	\end{equation}
	where $CD_{\text{RBM}}$ is the classical unsupervised RBM cost function trained with the contrastive divergence learning procedure \cite{Hinton_guia}, and $\mathbf{W}$ is the matrix that has $\mathbf{w}_j$, with $j\in[1:m]$, as columns. In order to compute the gap we add to the output of the last RBM layer a soft-max regression decoder  trained during $500$ epochs separately. Several authors have combined RBMs with soft-max regression, among them \cite{Hinton2006_DBN}, \cite{DBLP:conf/uai/SrivastavaSH13} and \cite{Chopra2018}. 
	
	Following suggestions from \cite{Hinton_guia}, we study the Lagrange multiplier when $\lambda\geq0.00001$. We choose learning rates in $0.1$, a batch-size of $100$, $256$ hidden units and we train during $200$ epochs. We start with a momentum of $0.5$ and change to $0.9$ after $5$ epochs. 
\end{itemize}

The following subsections present new experiment using CIFAR-10 dataset: Comparison of the behavior of the error gap and mutual information, and an example about the tradeoff between $\epsilon(K)$ and $r(K)$. 

\subsection{Comparison between the behavior of the error gap and mutual information}

\begin{figure}
	\centering
	\begin{subfigure}{.32\textwidth}
		\centering
		\includegraphics[scale=0.24]{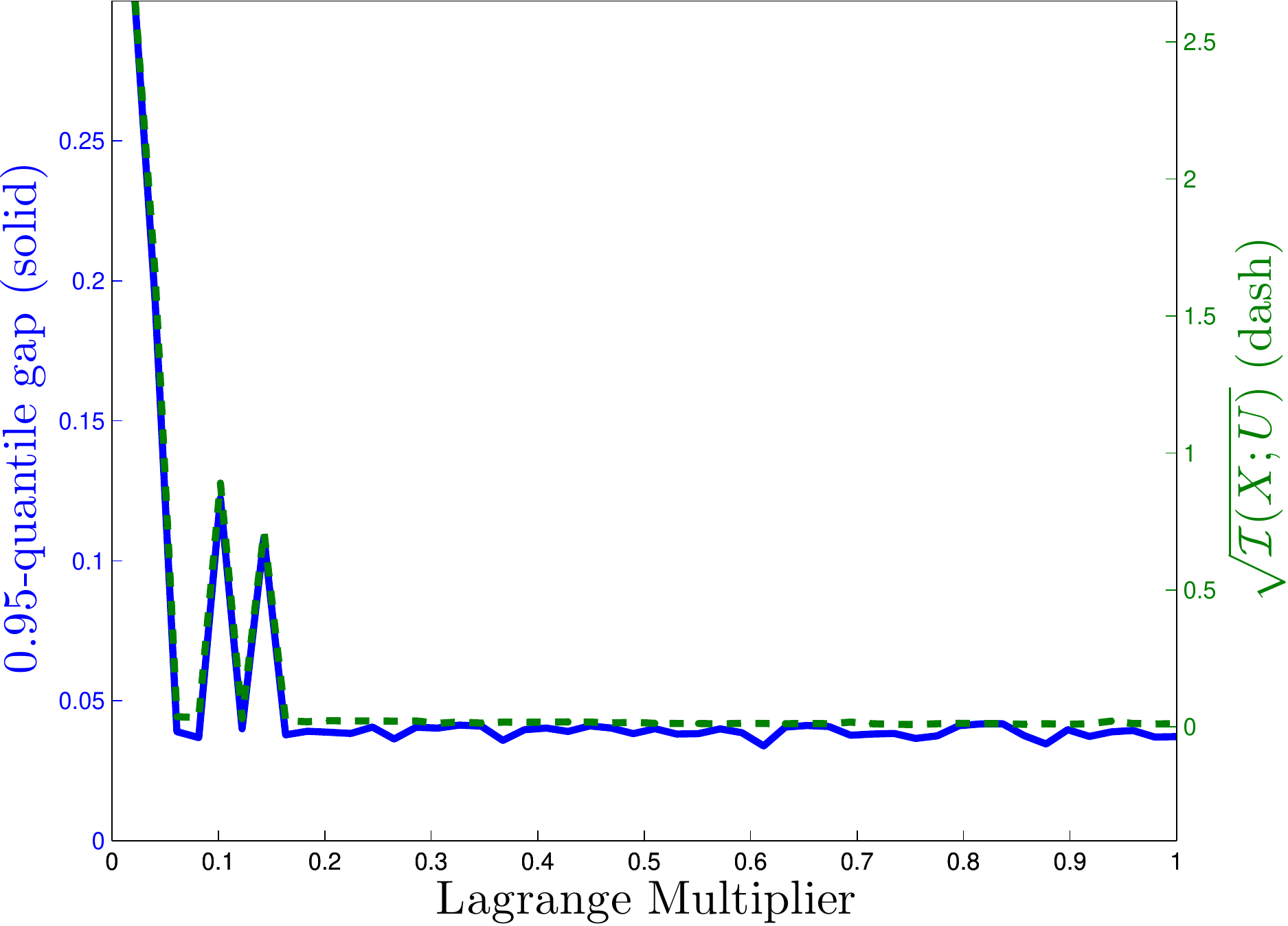}
		\caption{}
		\label{fig:normalcifar}
	\end{subfigure}
	\begin{subfigure}{.32\textwidth}
		\centering
		\includegraphics[scale=0.24]{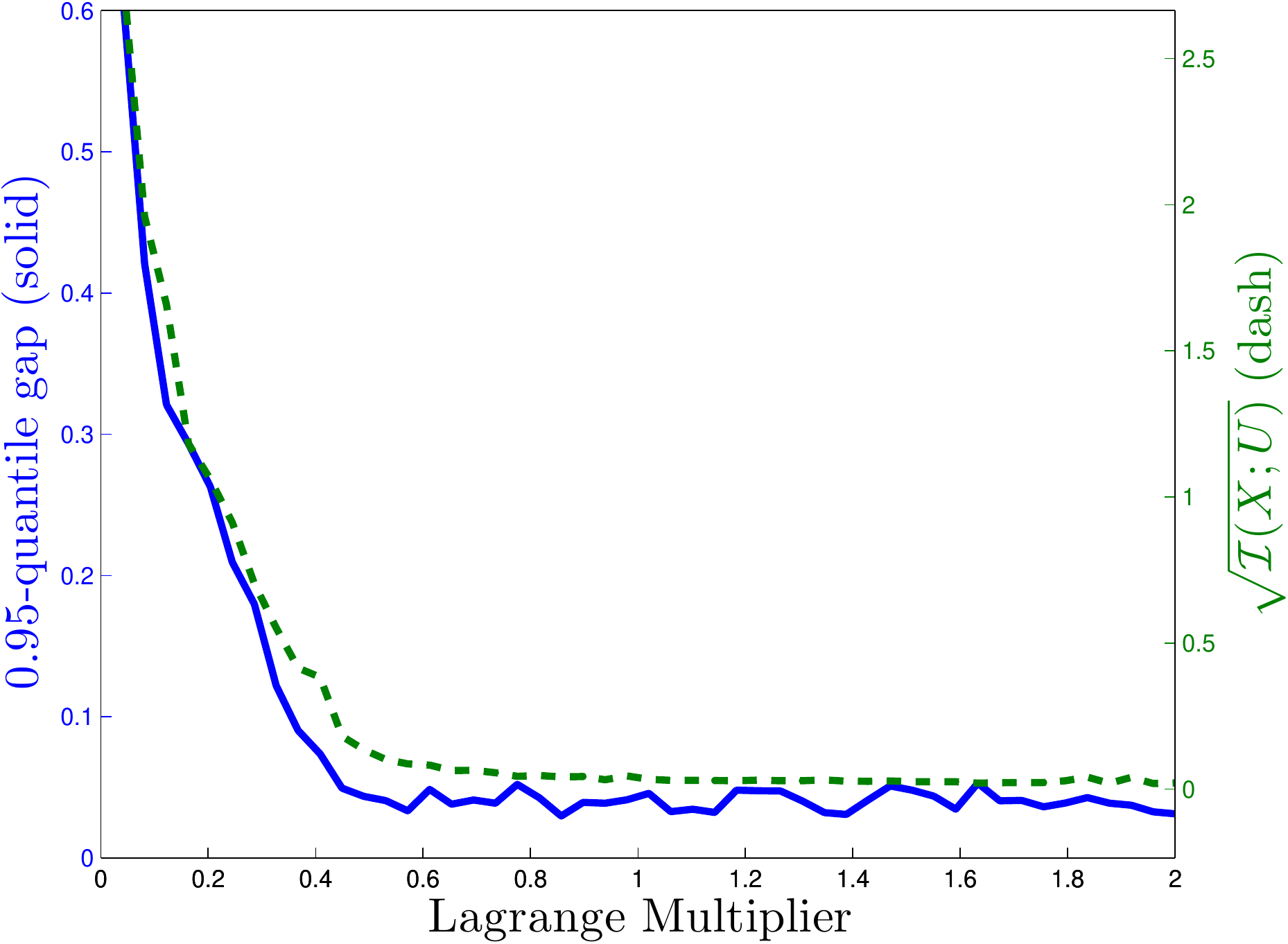}
		\caption{}
		\label{fig:lognormalcifar}
	\end{subfigure}
	\begin{subfigure}{.32\textwidth}
	\centering
	\includegraphics[scale=0.24]{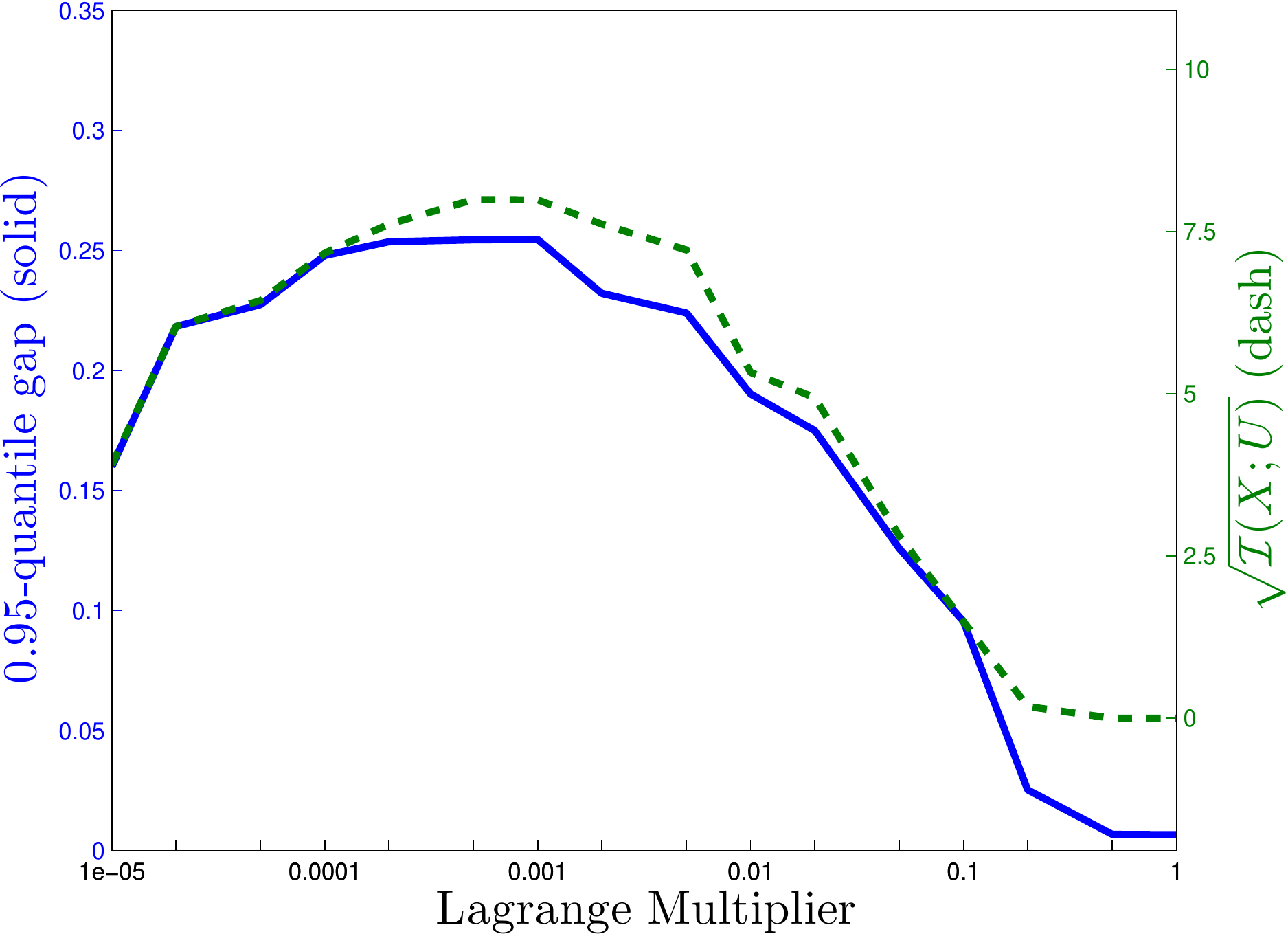}
	\caption{}
	\label{fig:rbmcifar}
	\end{subfigure}
	\caption{Comparison between $0.95$-quantile of $\mathcal{E}_{\textrm{gap}}(q_{U|X},Q_{\hat{Y}|U},\mathcal{S}_n)$ and the mutual information variational bound \eqref{eq:inf_radius_opt} for CIFAR-10 dataset and: (a) Normal encoder, (b) LogNormal encoder, (c) RBM encoder.}
\end{figure}

In this section we repeat simulation of Section \ref{sec:experimentos} for CIFAR-10 dataset, testing with images without disturbing. Figures \ref{fig:normalcifar}, \ref{fig:lognormalcifar} and \ref{fig:rbmcifar} show $0.95$-quantile  of $\mathcal{E}_{\textrm{gap}}(q_{U|X},Q_{\hat{Y}|U},\mathcal{S}_n)$ and the mutual information variational bound \eqref{eq:inf_radius_opt} as a function of the Lagrange multiplier $\lambda$ for Normal encoder, logNormal encoder and RBM encoder respectively. In these cases, mutual information and error gap have a behavior closer than MNIST.

\subsection{Discretization Tradeoff with CIFAR-10 dataset}

We use the setup presented in Section \ref{sec:normal_encoder} to implement numerically this tradeoff between $\epsilon(K)$ and $r(K)$ through K-Means algorithm. For every $K$, we iterate between:
\begin{itemize}
	\item \textbf{The samples coloring}: Let the loss centroids $\{\ell(x^{(k)},y)\}_{k=1}^K$ for each $y\in\mathcal
	Y$, we assign $x_i\in\mathcal{K}_{k_i}$, where $k_i$ is computed as
	\begin{equation}
	k_i=\arg\min_{1\leq k\leq K}\max_{y\in\mathcal{Y}}\left|\ell(x_i,y)-\ell(x^{(k)},y)\right|;
	\end{equation}
	\item \textbf{Find centroids}: We compute for each\footnote{We compute the loss centroid inside the true centroid $x^{(k)}$.} $k$ and each $y\in\mathcal{Y}$
	\begin{equation}
	\ell(x^{(k)},y)=\frac{1}{|\{i:\;x_i\in\mathcal{K}_k|\}}\sum_{\substack{i:\\x_i\in\mathcal{K}_k}}\ell(x_i,y).
	\end{equation}	
\end{itemize}
After that, we estimate $\epsilon(K)$ and $r(K)$ as:
\begin{equation}
\epsilon(K)=\max_{1\leq i\leq n}\max_{y\in\mathcal{Y}}\left|\ell(x_i,y)-\ell(x^{(k_i)},y)\right|,\quad r(K)=\frac{1}{\displaystyle\min_{1\leq k \leq K}\frac{|\{i:\;x_i\in\mathcal{K}_k|}{n}},
\end{equation}
using a $5K$ dataset of CIFAR-10. This algorithm focus on minimize $\epsilon(K)$ inside the tradeoff, and this function is really sensible to the samples (it only depends of the maximum). Fig. \ref{fig:vsk} shows a noisy tradeoff, which achieves its minimum at $K=16$.
\begin{figure}
	\centering
	\includegraphics[scale=0.5]{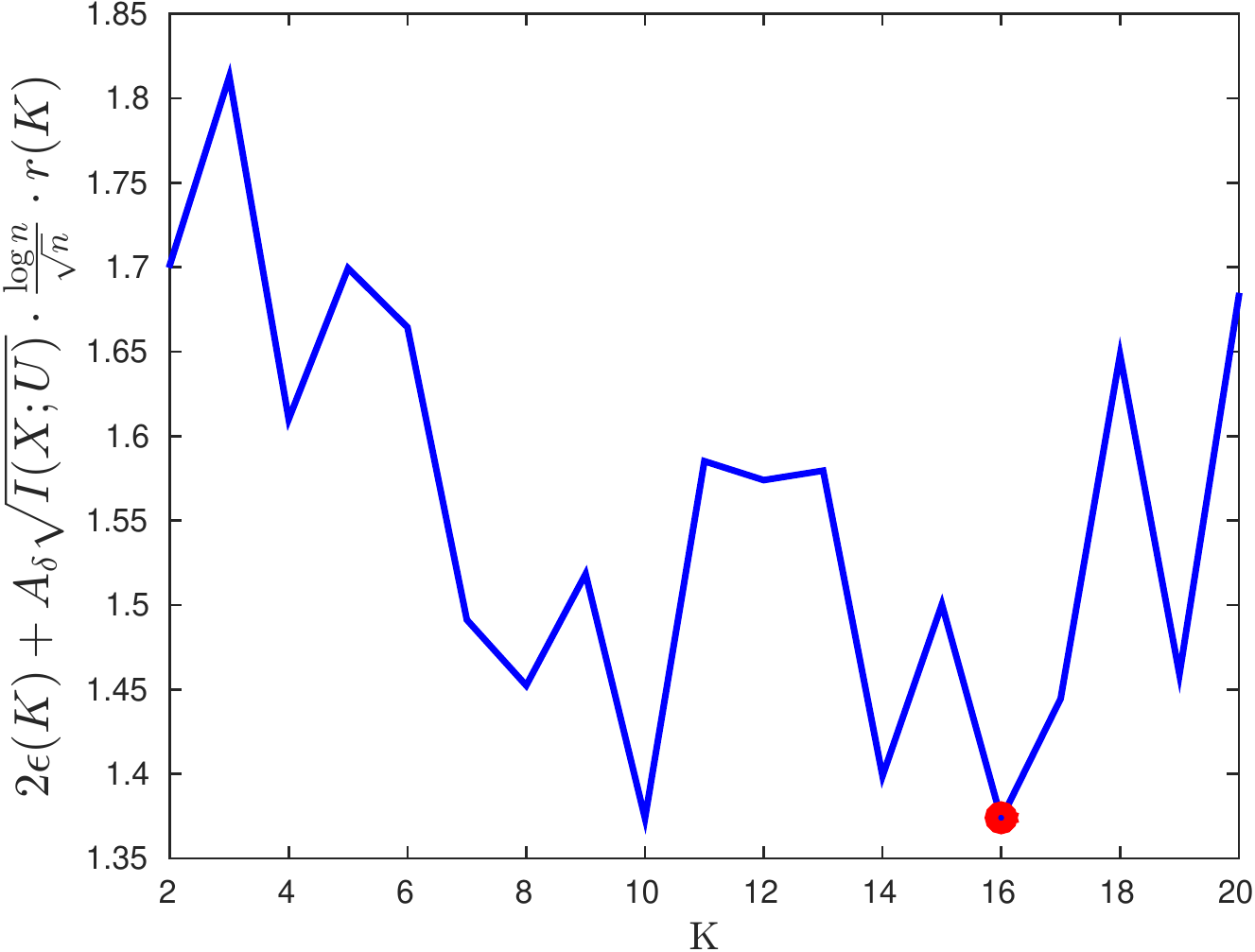}
	\caption{Tradeoff between $\epsilon(K)$ and $r(K)$ vs the number of cells, highlighting its minimum value, in a Normal Encoder example with CIFAR-10 dataset.}
	\label{fig:vsk}
\end{figure}

\section{Auxiliary Results}\label{app:auxiliary-results}

In this appendix some auxiliary facts, which are used in proof of the main results, are listed without proof.

\begin{lemma}[Theorem 11.2.1 \citep{cover}]\label{lem:diverg}
	Let ${P}\in\mathcal{P}(\mathcal{X})$ be a discrete probability distribution and let $\hat{P}$ be its empirical estimation over a $n$-data set $\mathcal{S}_n$. Then, 
	\begin{equation}
	\mathcal{D}(\hat{P}\|P)\leq |\mathcal{X} | \frac{\log(n+1)}{n}+\frac{1}{n}\log(1/\delta)
	\end{equation}
	with probability at least $1-\delta$ over the choice of $\mathcal{S}_n$.
\end{lemma}

\begin{lemma}[Union bound application]\label{lem:simultaneidad}
	Let $\{\mathcal{A}_k\}_{k=1}^m$ events such that $\Pr(\mathcal{A}_k)\geq1-\delta$ for each $k\in[1: m]$. Then, $\mathbb{P}\left(\bigcap_{k=1}^m \mathcal{A}_k\right)\geq1-\delta m$.
\end{lemma}

\begin{lemma}[Application of McDiarmid's Inequality to the vector probability]\label{lem:pvector}
	Let ${P}\in\mathcal{P}(\mathcal{X})$ be any probability distribution and let $\hat{{P}}$ be its empirical estimation over a $n$-data set $\mathcal{S}_n$. Then, with probability at least $1-\delta$ over $\mathcal{S}_n$:
	\begin{equation}
	\|\mathbf{P}-\mathbf{\hat{{P}}}\|_2\leq\frac{1+\sqrt{\log(1/\delta)}}{\sqrt{n}}.
	\end{equation}
\end{lemma}

\begin{lemma}[Adaptation of \cite{Shamir:2010:LGI:1808343.1808503}]\label{lem:Hu}
	Consider the encoder given by $q_{U|X}$. We have
	\begin{equation}
	\left| \mathcal{H}_d(q^D_U)- \mathcal{H}_d(\hat{q}^D_U)\right|\leq\int_\mathcal{U}\phi\left(\|\mathbf{P}_X-\mathbf{\hat{P}}_X\|_2 \sqrt{\mathbb{V}\big( \mathbf{q}_{U|X}(u|\cdot) \big) }\right)du 
	\end{equation}
	with
	\begin{equation}
	\phi(x)=\left\{\begin{array}{cc}0&x\leq0\\-x\log(x)&0<x<e^{-1}\\e^{-1}&x\geq e^{-1}
	\end{array}\right.
	\label{eq:cotaphi}
	\end{equation}
	for $\|\mathbf{P}_X-\mathbf{\hat{P}}_X\|_2$ not so small.
\end{lemma}

\begin{lemma}[Adaptation of \cite{Shamir:2010:LGI:1808343.1808503}]\label{lem:Hu|y} 
	Let $\mathcal{U}$ a compact space. Consider the encoder $q_{U|X}$, then	
	\begin{align}
	\left| \mathcal{H}_d(q^D_{U|Y}|P_Y)- \mathcal{H}_d(\hat{q}^D_{U|Y}|\hat{P}_Y)\right|&\leq \|\mathbf{P}_Y\!-\!\mathbf{\hat{P}}_Y\|_2\sqrt{|\mathcal{Y}|}\log\text{Vol}\left(\mathcal{U}\right)\\
	&\hspace{-2.5cm}+\mathbb{E}_{P_Y}\!\left[ \int_\mathcal{U}\!\phi\left(\left\|\mathbf{P}_{X|Y}(\cdot|Y)\!-\!\mathbf{\hat{P}}_{X|Y}(\cdot|Y)\right\|_2\!\sqrt{\mathbb{V}\big( \mathbf{q}_{U|X}(u|\cdot )  \big) }\right)du\right],\nonumber
	\end{align}
	for $\max_y\|\mathbf{P}_{X|Y}(\cdot|y)-\mathbf{\hat{P}}_{X|Y}(\cdot|y)\|_2$ not so small
\end{lemma}

\begin{lemma}[\cite{Shamir:2010:LGI:1808343.1808503}]\label{lem:cotaphi}
	Let $n\geq a^{2}e^{2}$, then $	\phi\left(\frac{a}{\sqrt{n}}\right)\leq\frac{a}{2}\frac{\log(n)}{\sqrt{n}}+\frac{e^{-1}}{\sqrt{n}}$.
\end{lemma}

\end{document}